\documentclass[letterpaper]{article}
\usepackage{aaai2026} % DO NOT CHANGE THIS
\nocopyright    
\usepackage{times}  % DO NOT CHANGE THIS
\usepackage{helvet}  % DO NOT CHANGE THIS
\usepackage{courier}  % DO NOT CHANGE THIS
\usepackage[hyphens]{url}  % DO NOT CHANGE THIS
\usepackage{graphicx} % DO NOT CHANGE THIS
\usepackage{natbib}  % DO NOT CHANGE THIS AND DO NOT ADD ANY OPTIONS TO IT
\usepackage{caption} % DO NOT CHANGE THIS AND DO NOT ADD ANY OPTIONS TO IT
\usepackage{algorithm}
\usepackage{amsmath}
\usepackage{amssymb}
\usepackage{booktabs}
\usepackage{listings}
\usepackage{multirow} 
\newtheorem{definition}{Definition}
\newtheorem{theorem}{Theorem}
\newtheorem{lemma}{Lemma}
\newtheorem{example}{Example}
\newenvironment{proof}{\par\noindent\textbf{Proof.}\ }{\hfill$\square$\par}
\usepackage{rotating}
\usepackage{scalefnt}
\usepackage{upgreek}
\usepackage{enumitem}
\usepackage{array}
\usepackage{algorithmicx}
\usepackage{algpseudocode}

\usepackage{newfloat}
\usepackage{listings}
\DeclareCaptionStyle{ruled}{labelfont=normalfont,labelsep=colon,strut=off} % DO NOT CHANGE THIS
\floatstyle{ruled}
\newfloat{listing}{tb}{lst}{}
\floatname{listing}{Listing}
\title{SITA: A Framework for Structure-to-Instance  Theorem Autoformalization}
\author {
    Chenyi Li\textsuperscript{\rm 1},
    Wanli Ma\textsuperscript{\rm 2},
    Zichen Wang\textsuperscript{\rm 1},
    Zaiwen Wen\textsuperscript{\rm 2}\thanks{Corresponding author}
}
\affiliations {
    \textsuperscript{\rm 1}School of Mathematical Sciences, Peking University\\
    \textsuperscript{\rm 2}Beijing International Center for Mathematical Research, Peking University\\
    lichenyi@stu.pku.edu.cn, wlma@pku.edu.cn, princhernwang@gmail.com,
    wenzw@pku.edu.cn
}

\usepackage{bibentry}
\begin{document}

\maketitle

\begin{abstract}
While large language models (LLMs) have shown progress in mathematical reasoning, they still face challenges in formalizing theorems that arise from instantiating abstract structures in concrete settings. With the goal of auto-formalizing mathematical results at the research level, we develop a framework for structure-to-instance theorem autoformalization (SITA), which systematically bridges the gap between abstract mathematical theories and their concrete applications in Lean proof assistant. Formalized abstract structures are treated as modular templates that contain definitions, assumptions, operations, and theorems. These templates serve as reusable guides for the formalization of concrete instances. Given a specific instantiation, we generate corresponding Lean definitions and instance declarations, integrate them using Lean’s typeclass mechanism, and construct verified theorems by checking structural assumptions. We incorporate LLM-based generation with feedback-guided refinement to ensure both automation and formal correctness. Experiments on a dataset of optimization problems demonstrate that SITA effectively formalizes diverse instances grounded in abstract structures.
\end{abstract}

\begin{links}
    \link{Code}{https://github.com/chenyili0818/SITA}
   % 
   % \link{Extended version}{https://aaai.org/example/extended-version}
\end{links}

\section{Introduction}
Recent advances in large language models (LLMs) have demonstrated impressive capabilities in solving mathematical problems and generating natural language proofs \cite{ahn-etal-2024-large, welleck2021naturalproofs}. However, such informal outputs often lack the formal rigor required for verification by proof assistants. To address this limitation, interactive theorem provers such as Lean \cite{de2015lean}, Coq \cite{huet1997coq}, and  Isabelle \cite{Nipkow2002APA} have been developed to rigorously validate each step of a proof. These systems enhance the soundness and reliability of machine-generated mathematical reasoning \cite{yang2024formal}. 

Several recent efforts have explored the use of LLMs to construct formal proofs from given formal statements. Two main paradigms have emerged: stepwise generation and whole-proof generation. The stepwise approach, exemplified by Leandojo \cite{yang2023leandojo}, decomposes the proof process into premise selection and tactic generation, modeling proof search as a sequence of local decisions. BFS-prover \cite{xin2025bfsprover} utilizes best-first tree search in generation. In contrast, whole-proof generation attempts to synthesize an entire proof in a single pass, typically by first producing a rough draft and then refining it into a valid proof \cite{jiang2023draft}. These approaches often integrate natural language reasoning with formal symbolic reasoning, and are further improved through reinforcement learning techniques via expert iterations \cite{xin2024deepseekprover, ren2025deepseekprover, wang2025kimina, lin2025goedel}.

In parallel, the task of autoformalization, i.e. translating mathematical problems stated in natural language into formal statements, has also received increasing attention. Early approaches leverage few-shot prompting \cite{Brown2020language} to perform this translation \cite{wu2022autoformalization, azerbayev2023proofnet}. Building on this foundation, subsequent work improves output quality through sampling strategies that select optimal outputs from multiple generations \cite{li2024autoformalize, poiroux2025improvingautoformalizationusingtype}, or by retrieving related theorems from formal libraries \cite{liu2025rethinking}. While these methods are relatively easy to implement, their effectiveness remains limited by the capabilities of the underlying language model. Training-based approaches have also been explored for statement formalization \cite{jiang2023multilingual}. The Lean Workbook \cite{Leanworkbook2025Ying} introduces an iterative data generation and filtering pipeline to formalize problems from online math forums in Lean 4. ATLAS \cite{liu2025atlasautoformalizingtheoremslifting} further improves formalization performance through expert iteration and knowledge distillation.

\begin{figure*}[htbp]
  \centering
  \includegraphics[width=\textwidth,height=0.555\textwidth]{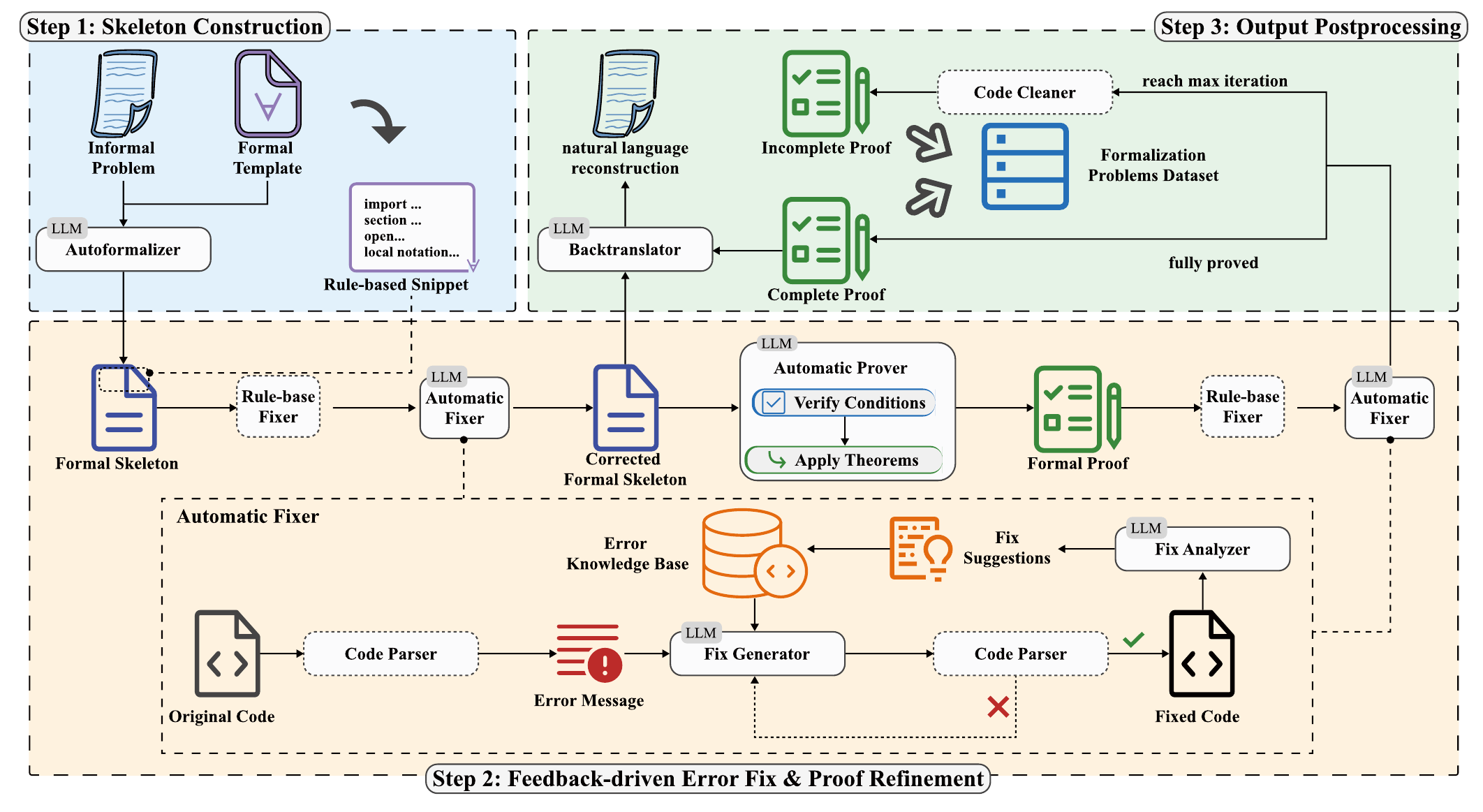} 
  \caption{Overall pipeline of SITA}
  \label{fig:pipeline}
\end{figure*}

In this paper, we focus on a specific and widely used form of mathematical reasoning: deriving concrete instances from abstract structures. This reasoning pattern involves applying abstract theorems, defined over general mathematical structures, such as convergence guarantees of gradient descent on convex functions, to derive instance-specific results, such as the convergence of gradient descent for ridge regression. While this form of reasoning might seem straightforward, verifying that a specific instance satisfies the assumptions of an abstract theorem often demands non-trivial effort. This process may involve reformulating definitions, proving auxiliary lemmas, or performing symbolic manipulations, and in some cases, it occupies entire sections of theoretical papers \cite{Duan2020Adaptive}. Although this reasoning paradigm is common in algorithmic mathematics and research-level papers \cite{Bao_2014_CVPR}, it heavily relies on expert intuition and manual work. It has not been systematically explored in the context of automated formalization yet. Our goal is to automate this process using the Lean proof assistant by translating natural language statements into formal definitions, verifying structural assumptions, and applying general theorems to construct formal proofs.

While most prior formalization research has focused on Olympiad-style mathematical problems, our study targets a broader and practically important form of reasoning: structure-to-instance theorems, with the goal of auto-formalizing mathematical results at the research level. We introduce SITA, an end-to-end automated pipeline (see Figure \ref{fig:pipeline}), which leverages LLMs to formalize structure-to-instance reasoning end-to-end. Given a formalized abstract structure and an informal instance description, the pipeline generates instance-specific formal definitions, verifies structural conditions, and applies the abstract result to produce a complete formal proof. This pipeline enables the reliable reuse of high-level mathematical theories in domain-specific applications, supporting the construction of modular and verifiable formal reasoning workflows. Our framework is applicable to domains such as optimization, signal processing, and algebraic computation, and contributes toward the development of scalable, theorem-grounded formal libraries.
% The pipeline first formalizes instance-specific definitions, then verifies that the required structural assumptions are satisfied, and finally applies the abstract theorem to derive the corresponding concrete result in a proof assistant. This approach enables the rigorous reuse of high-level mathematical theories in applied settings, facilitating modular, scalable, and verifiable reasoning, which is common in realms including optimization, signal processing, abstract algebra, etc. By formalizing structure-to-instance reasoning, we enable the systematic transfer of abstract mathematical results into concrete, domain-specific settings. This supports the development of large-scale, reusable formal libraries and lays the groundwork for rigorous, verifiable pipelines in both theoretical and applied computational research.

Our main contributions are as follows.

1) We define the structure-to-instance reasoning task by identifying relevant mathematical structures, their concrete instances, and the key sub-tasks involved to apply general results to specific problems. 

2) We propose SITA, a framework that leverages LLMs to automatically generate formal definitions, lemmas, and theorems for instance-specific problems. An iterative refinement and error-feedback mechanism--guided by a growing knowledge base--enhances both the reasoning accuracy and the quality of generated formal code.

3) We construct a benchmark of optimization problems grounded in real-world applications to evaluate the proposed framework. Experimental results demonstrate that SITA successfully generates formal definitions, instantiates abstract theorems, and performs assumption verification.

\section{The Structure-to-Instance Problem}
\label{sec:problem}

\subsection{Mathematical Structures and Their Instances}
We mainly focus on operations defined over abstract mathematical structures. Following the approach \cite{baanen2022use} of bundling parameters and conditions into unified structures, a mathematical structure together with its associated operations is defined as follows.

\begin{definition}[Mathematical Structures with Operations]
A \emph{mathematical structure} is a four-tuple
\[
  \mathcal{S} = \bigl\langle
    \mathcal{D},\;
    \mathcal{O},\;
    \mathcal{C},\;
    \mathcal{T}
  \bigr\rangle.
\]
\begin{itemize}
\item $\mathcal{D}$ (\emph{Definitions}): Primitive notions or axioms introducing the fundamental objects of the theory.
\item $\mathcal{O}$ (Operations): Computational procedures, symbolic operations, or constructive rules built on $\mathcal{D}$, serving as the algorithmic or operational core of the structure. These may include algorithms, transformation rules, or abstract procedures defined over the definitions.
\item $\mathcal{C}$ (\emph{Conditions}): Assumptions imposed on elements in $\mathcal{D}$ and $\mathcal{O}$ to guarantee the correctness or performance.
\item $\mathcal{T}$ (\emph{Theorems}): Theorems and lemmas concerning the behavior of $\mathcal{O}$, derived from the definitions in $\mathcal{D}$ under the assumptions in $\mathcal{C}$.
\end{itemize}

\end{definition}

Such abstract structures are ubiquitous in mathematics. We provide an illustrative example below, with additional cases in the Appendix B. Operations as optimization algorithms are mainly studied in this paper.
\begin{example}
The gradient descent method for unconstrained optimization can be expressed as
\(
  \mathcal{S}_{\mathrm{GD}}
  = \bigl\langle
    \mathcal{D},\,\mathcal{O},\,\mathcal{C},\,\mathcal{T}
  \bigr\rangle.
\)
\begin{itemize}
  \item $\mathcal{D}$: The unconstrained optimization problem is defined as
    $\min_x f(x)$, where $ f: \mathbb{R}^n\to\mathbb{R}$.
  \item $\mathcal{O}$: The update scheme of the gradient descent method to solve the problem is
   \( x_{k+1}=x_k-\eta_k\nabla f(x_k).\)
  \item $\mathcal{C}$: Function $f(x)$ is smooth and convex; $\nabla f$ is $L$‐Lipschitz continuous; $0<\eta_k<1/L$.
  \item $\mathcal{T}$: A collection of theoretical results and guarantees related to the operations in $\mathcal{O}$, including:
\begin{itemize}
    \item the descent lemma;
    \item sublinear convergence theorems.
\end{itemize}
\end{itemize}
\end{example}

A concrete instance $\mathbb{I}$ of an abstract structure $\mathcal{S}$ is formed by specifying concrete realizations of its components, the definitions $\mathcal{D}$ and operations $\mathcal{O}$, within a specific problem domain. Although the abstract structure remains fixed, it can generate a wide range of instances that share the same formal framework but differ in the functions or variables used. Applying $\mathcal{S}_{\mathrm{GD}}$ to logistic regression or least-squares problems produces distinct instances with different objectives, while preserving similar structural properties.

Theoretical results in $\mathcal{T}$ can be applied to a concrete instance $\mathbb{I}$, provided that the conditions in $\mathcal{C}$ are satisfied. Verifying these conditions is often simpler than proving the results directly, as it only requires checking the set of assumptions. However, some verifications, such as the Kurdyka–Łojasiewicz (KL) property \cite{bolte2014proximal} in nonconvex optimization, can be challenging. The use of abstract structures enables the modular reuse of theoretical results across diverse instances, supporting the systematic generation of theoretical results.

\subsection{Autoformalization Targets}

Given a formally defined abstract structure $\mathcal{S}$ and a natural language description of a concrete instance where $\mathcal{S}$ applies, we consider the task of auto-formalizing such instances $\mathbb{I}$. Each instance includes its own problem data, operation and the corresponding theorems. We call this process the autoformalization of structure-to-instance theorems. The main idea is to use the formalized structure to guide the formalization of instances in application domains, allowing theoretical results to be reused. We use Lean as the theorem prover in this paper. A basic introduction to formalization using Lean is provided in Appendix A.

The structure-to-instance autoformalization paradigm consists of the following two steps. An illustration is also provided in Figure \ref{fig:illustration}.

\begin{enumerate}

\item \textbf{Instance specification and structural integration.}
We begin by formally specifying the instance-level problem data and associated operations, instantiating the abstract components $\mathcal{D}$ and $\mathcal{O}$. The instance is then integrated into the abstract framework $\mathcal{S}$ through \texttt{instance} declarations in Lean, which certify that the concrete problem provides definitions, operations, and structural properties compatible with those of $\mathcal{S}$. This alignment enables the generic definitions, operations, and theorems of $\mathcal{S}$ to be uniformly applied to the concrete setting.

\item \textbf{Theorem instantiation and verification.}
Once structural alignment is established, we verify that the instance satisfies the conditions in $\mathcal{C}$, concrete assumptions on the data and operations required for sound application of the theorems in $\mathcal{T}$. After proving these assumptions, we instantiate the abstract theorems in $\mathcal{T}$ to derive concrete results, thus avoiding redundant proof efforts.

\end{enumerate}

% \begin{enumerate}
% \item \textbf{Formal definition of instance structures and operations.}
% The first step is to formally define the instance-specific problem data and associated operation. These definitions instantiate the abstract components $\mathcal{D}$ and $\mathcal{O}$ of the general structure, forming a concrete formal problem class tailored to the target application.

% \item \textbf{Structure alignment via instance declarations.}
% Concrete problems are connected to the abstract structure $\mathcal{S}$ through \texttt{instance} declarations in Lean. These certify that the problem satisfies the structural requirements of $\mathcal{S}$, enabling the uniform application of its definitions, operations, and theorems to the concrete setting.

% \item \textbf{Theorem instantiation and proof.}
% In the final step, one verifies the required assumptions in $\mathcal{C}$ for the concrete instance. Once these conditions are established, the desired theorem can be derived from instantiating the abstract theorems in $\mathcal{T}$, yielding the corresponding results without re-proving them from scratch.

% \end{enumerate}

The structure-to-instance formalization paradigm embodies a central pattern in mathematical reasoning: once a property is established under general assumptions, it can be systematically instantiated across all problems that meet those conditions. Our framework systematizes this reasoning at scale, automating the process of generating verified proofs for a wide variety of instances from a single abstract theory.

% To our knowledge, this direction has not been fully explored in prior autoformalization work. Existing efforts primarily address the challenge of interpreting informal input; in contrast, we tackle the algorithmic generation of formal proofs for concrete instances, based on already formalized abstract structures. This structured instantiation approach is especially well-suited for building scalable and reusable formal libraries.

This task fundamentally differs from traditional autoformalization efforts, which primarily focus on translating informal theorems into formal statements. Such approaches often neglect the formalization of definitions and concrete instances. In contrast, we consider a setting where the abstract structure $\mathcal{S}$ is already formalized, and the goal is to automate its instantiation: specializing defitions $\mathcal{D}$ and operations $\mathcal{O}$ to domain-specific elements, verifying that the structural conditions $\mathcal{C}$ are satisfied, and applying general theorems $\mathcal{T}$ to derive instance-level guarantees. This process enables the systematic reuse of abstract formal results across diverse application domains.

Moreover, this framework facilitates the creation of large, composable datasets of verified application theorems. These datasets support practical applications such as safety verification and certified manipulation, while also providing essential training data for future machine learning systems in formal mathematics, bridging symbolic reasoning and automated synthesis in a principled and extensible manner.

\begin{figure}
    \centering
    \includegraphics[width=\linewidth]{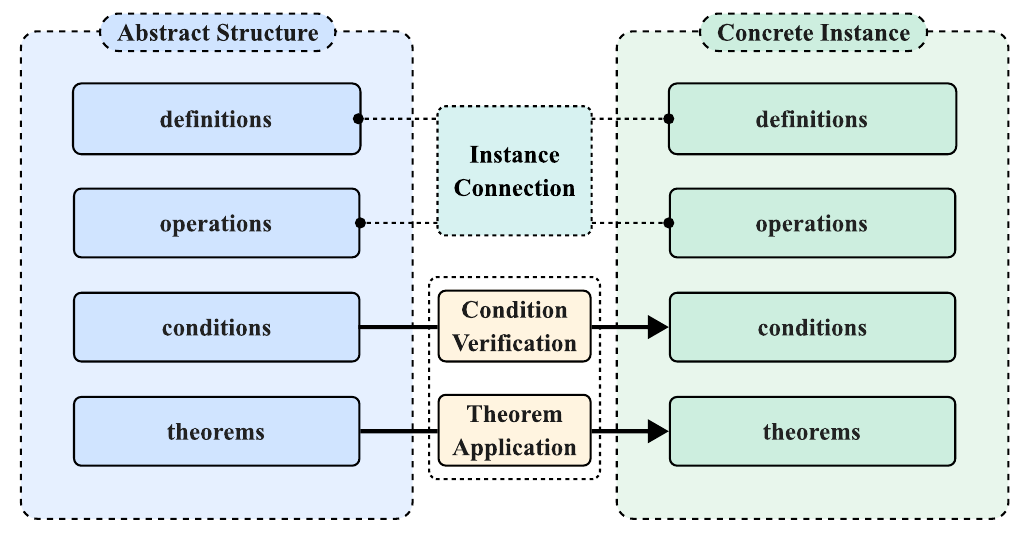}
    \caption{Illustration of structure-to-instance formalization.}
    \label{fig:illustration}
\end{figure}
% The formalization process begins by identifying a suitable abstract structure $\mathcal{S}$ in Lean, which encodes general definitions, algorithms, assumptions, and theorem patterns relevant to the target instance. This serves as the guiding framework for formal development.

% Next, domain-specific data is used to instantiate $\mathcal{D}$ and $\mathcal{O}$, generating concrete definitions and algorithms. The system fills in the instance formalization within the abstract template, producing Lean code that specifies the problem and theorems to verify.

% Finally, to derive the theorems in $\mathcal{T}$, the system verifies that the instance satisfies the required conditions $\mathcal{C}$. It generates and proves necessary supporting lemmas, completing the formal derivation and ensuring the correctness of the instance under the given assumptions.

% To the best of our knowledge, the structure-to-instance formalization chain targets a problem setting not addressed by existing work, the integration of autoformalization and automated theorem proving within a unified framework. It also offers several advantages over existing AI-assisted formalization efforts. The framework enables the systematic generation of new instances within structured, extensible formal libraries. By coupling symbolic reasoning with structured data inputs, it supports the construction of large-scale datasets of formalized algorithms and proofs—resources that are crucial for training next-generation automated theorem provers and formal synthesis systems.

{
\subsection{Example: Lasso}\label{example: Lasso}

To illustrate the structure-to-instance autoformalization paradigm shown in Fig.~\ref{fig:illustration}, we present an example involving the proximal gradient method \cite{Parikh2014Proximal} applied to the Lasso problem \cite{Tibshirani1996Regression}. We only show part of the code and highlight the main idea of our structure-to-instance formalization pipeline. The full code and detailed explanations are provided in the Appendix C. The code builds on Mathlib \cite{mathlibcommunity} and Optlib \cite{li2024formalization, li2025block}. 

\paragraph{1. Abstract Structure \(\mathcal{S}\)}
We first obtain the formalized definitions of the abstract structure. In this setting, $\mathcal{D}$ as the composite optimization problem \(\min_x \psi(x) = f(x) + h(x)\) is formalized as below.
\begin{lstlisting}
class composite_pro (f h : E → ℝ)
def composite_pro.tar (_ : composite_pro f h) := f + h
\end{lstlisting}
The corresponding optimization algorithm, the proximal gradient method, is given as follows.
\begin{lstlisting}
class pg (pro : composite_pro f h) (x₀ : E) := ...
\end{lstlisting}
Assuming appropriate conditions \(\mathcal{C}\) (e.g., convexity, Lipschitz smoothness), the convergence theorem \(\mathcal{T}\) is stated as:
\begin{lstlisting}
theorem pg_method_converge (conditions):
∀ k, (pro.tar (alg.x k) - pro.tar xm)
  ≤ 1 / (2 * k * alg.t) * ‖ x₀ - xm ‖ ^ 2 := 
\end{lstlisting}

\paragraph{2. Concrete Instance \(\mathbb{I}\): Lasso}

The Lasso problem minimizes $\frac{1}{2}\|Ax - b\|^2 + \mu\|x\|_1$. The proximal gradient method solves the Lasso problem with \(
(x_{k+1})_i = \operatorname{sign}(z_i) \cdot \max\{|z_i|-t\mu,0\}
\), where $z = x_k -t A^\top (Ax_k-b)$. We expect to auto-formalize the corresponding structures as:
\begin{lstlisting}
class Lasso_pro (A b μ) := ...
class pg_Lasso (pro : Lasso_pro A b μ) (x₀ : E) := ...
\end{lstlisting}

Instances below are used to link the Lasso problem class with the abstract class using Lean's instance mechanism, as illustrated in the central arrows of Figure~\ref{fig:illustration}:
\begin{lstlisting}
instance Lasso_pro.composite_pro : composite_pro f g := ...
instance pg_Lasso.pg : pg pro x₀ := {
  t := self.t, x := self.x, initial := ..., update := ...}
\end{lstlisting}

\paragraph{3. Theorem Transfer via Instance}

By verifying the assumptions in \texttt{pg\_method\_converge}, we obtain the Lasso-specific convergence result through reuse:
\begin{lstlisting}
theorem Lasso_convergence : ∀ (k : ℕ+), (pro.target (alg.x k) - pro.target xm) ≤ 1 / (2 * k * alg.t) * ‖ x₀ - xm ‖ ^ 2 := by sorry
\end{lstlisting}

This example demonstrates the complete mapping from an abstract structure \(\mathcal{S}\) to a concrete instance \(\mathbb{I}\) as shown in Figure~\ref{fig:illustration}. Instance declarations enable modular reuse of definitions and theorems, facilitating scalable formalization of optimization algorithms.
}

\section{The SITA Framework}
\label{sec:algo}

The \textbf{SITA} pipeline, illustrated in Figure~\ref{fig:pipeline}, provides an end-to-end framework for structure-to-instance theorem autoformalization. It transforms a natural language description of a concrete problem into a verified Lean file by aligning the underlying concepts with a formalized abstract structure. The pipeline comprises three main stages: (1) skeleton construction, (2) feedback-driven error fix and proof refinement, and (3) output postprocessing. Each stage integrates large language models\footnote{Detailed prompt templates are provided in Appendix D.2.} with Lean’s type system to ensure both automation and formal soundness.

\subsection{Skeleton Construction}

Given an \emph{informal problem description}, the process begins by identifying a suitable abstract structure, either automatically or via user input, that aligns with the concrete setting. Each abstract structure is represented as a reusable \emph{formal template} consisting of definitions, assumptions, and theorems. Using one-shot prompting, the system generates the core formal components of the instance, including problem-specific definitions, operation classes, and target theorem statements. It also produces \emph{instance lemmas} that declare Lean typeclass instances, connecting the concrete definitions to the abstract structure and certifying that the necessary structural conditions are satisfied. In addition, the system incorporates \emph{rule-based snippets} from the formal template, such as \texttt{import} statements, \texttt{section} headers, and \texttt{open} declarations, providing a structured formal context that guides and constrains the generation process. This helps improve correctness and ensures compatibility with existing formal libraries. The output of this stage is a \emph{formal skeleton}, i.e. a Lean file containing the complete problem setup and theorem declarations, prepared for subsequent refinement.

\subsection{Error Fix and Proof Refinement}
\paragraph{Error Fix}
To address errors in automatically generated Lean code, we employ a feedback-driven correction framework implemented in the \emph{error fix} stage. This component integrates two complementary steps: a \emph{rule-based fixer}, which applies deterministic edits targeting common syntactic and structural issues, and an \emph{automatic fixer}, which leverages Lean’s type-checking diagnostics and a self-updating \emph{error knowledge base} to suggest and validate adaptive repairs.

The first step is a static, rule-based correction module responsible for addressing common syntactic and structural issues. It performs a series of deterministic transformations, including formatting declarations, standardizing overloaded or non-canonical notations, removing unused hypotheses or empty section blocks, and enforcing conventions consistent with Lean’s standard libraries. The correction rules are implemented through symbolic rewriting and pattern matching, allowing efficient and general-purpose repair prior to more context-sensitive analysis.

Complementing the static step is an iterative feedback-driven correction module powered by Lean’s type checker and retrieval-augmented \cite{lewis2020retrieval, surveyRAG2024Fan} prompt construction. Given an error message $e$, the system queries an evolving error knowledge base $\mathcal{K}$, which stores structured correction strategies, illustrative fix examples, and references to relevant theorems or tactics. The retrieved entries $\mathcal{K}(e)$ are assembled into a custom prompt tailored to the specific error and passed to the language model. For instance, if a definition involves incorrect usage of a term $T$, the prompt includes the full statement of $T$ along with common usage patterns. When the error involves an undefined or non-canonical lemma, the retrieval module suggests similar alternatives from Lean’s libraries, which are automatically incorporated into the prompt context.

After each correction attempt, the updated code is recompiled and rechecked by Lean. If new errors are encountered, the feedback loop continues. For each resolved errors, the system logs the original error message, faulty code, and corrected code. These logs are processed by a secondary model to generate new \emph{fix suggestions}. Both the logs and the suggestions contribute to updating $\mathcal{K}$. This feedback mechanism progressively improves correction capabilities by capturing and generalizing emerging error patterns. The iteration terminates once the file type-checks successfully or a predefined retry limit is reached. 

This hybrid mechanism, combining symbolic rewriting, retrieval-augmented generation, and dynamic knowledge adaptation, yields a robust correction system capable of transforming flawed or incomplete formalizations into valid Lean code. The iterative process can be viewed as an explicit form of chain-of-thought reasoning \cite{Wei2022CoT} in the context of autoformalization. By treating the sequence of errors, prompts, and corrections as a structured evolving context, the system enables the model to reflect on prior failures and refine its responses accordingly, resulting in more accurate and coherent formalizations.

\paragraph{Proof Refinement}

After obtaining correct formal definitions and instance lemma statements, the next objective is to eliminate all remaining \texttt{sorry} placeholders by constructing valid Lean proofs. To this end, a whole-proof generation pipeline is employed. For each \texttt{sorry}, the system extracts the corresponding local environment, including hypotheses, definitions, and contextual information, and feeds it into the language model. Then, the model attempts to generate a proof term appropriate for the given goal.

In most cases, the generated proof attempts to verify that a concrete instance satisfies the assumptions of an abstract lemma or theorem. While these subgoals are typically less complex than proving the full statement, they can still be challenging, particularly in domain-specific contexts. To address this, the system maintains a retry counter to regulate repeated attempts. The error knowledge base is also leveraged at this stage to guide correction when a generated proof fails to type-check. Relevant proof construction suggestions, known tactic patterns, and examples of common pitfalls are retrieved and integrated into the prompt to improve model robustness. This mechanism reuses the infrastructure developed in the error correction stage, adapted here for proof synthesis. If the maximum number of attempts is reached without success, the pipeline falls back to a partial completion strategy: the language model attempts to construct as much of the proof as possible, retaining \texttt{sorry} placeholders for unresolved fragments.

\subsection{Output Postprocessing}
In the final stage, the system compiles the formalization results into a clean, verifiable Lean file, free of type errors, extending the approach of \cite{ospanov2025apollo} to handle both statements and proofs. For files that still contain errors, a hybrid postprocessing strategy is applied. First, rule-based techniques are utilized to patch the proof context by inserting appropriate \texttt{sorry} placeholders where necessary. If this fails to produce a well-typed file, an LLM-based fallback rewrites the file into a harmless, error-free version. If all proof goals are successfully discharged, the system outputs a complete formal artifact. Otherwise, any remaining incomplete components are explicitly marked and logged for downstream analysis or future refinement.

To support interpretability and broader applicability, SITA integrates a \emph{backtranslator} module that maps formal Lean constructs, such as definitions, assumptions, and theorems, back into natural language. This step facilitates human verification, improves documentation, and enables data augmentation. The \emph{reconstruction} outputs are aligned with the original informal problem descriptions, laying the groundwork for training bi-directional models that bridge formal and informal mathematical reasoning.

All intermediate artifacts, including incomplete proofs, corrected code fragments, Lean error traces, and natural language reconstruction outputs, are stored in a structured dataset. This archive supports fine-grained evaluation of model performance and serves as a targeted training resource for advancing proof generation, error correction, and formal-informal alignment in future models.

% \subsection{Advantages}
% \paragraph{Advantages}
% \begin{itemize}
%   \item Leverages Lean’s feedback to avoid large, uncheckable chunks.
%   \item Ensures each step type‐checks before proceeding.
% \end{itemize}

\section{Numerical Experiments}
\label{sec:experiments}

\subsection{Experimental Setup}
To assess the effectiveness of the proposed SITA paradigm, we conduct experiments on a collection of research-level mathematical problems that naturally conform to the structure-to-instance formalization setting. A particularly representative class of such problems arises in optimization, where many concrete applications can be viewed as instantiations of abstract algorithmic frameworks. By formalizing these instances, we demonstrate that SITA can correctly and efficiently generate Lean definitions and proofs that link concrete cases with their underlying abstract structures. While our current experiments focus on optimization, the proposed framework is not domain-specific. It can be extended to other areas of mathematics where formal abstraction and reusable operational structure play a central role.

\paragraph{Dataset} We construct a benchmark dataset comprising 42 representative optimization problems collected from widely used textbooks in numerical optimization. The problems span a broad spectrum of settings, including both constrained and unconstrained formulations, as well as convex and nonconvex objectives, thereby covering key categories commonly studied in the optimization literature. Each problem in the dataset is selected for its reliance on a standard optimization algorithm, including gradient descent (GD), proximal gradient (PGM), Nesterov’s acceleration \cite{nesterov1983method}, block coordinate descent (BCD) \cite{bolte2014proximal}, and the alternating direction method of multipliers (ADMM) \cite{Fazel2013hankelmatrix}. These algorithms have been formally verified in the Lean library Optlib, and we build on their existing formalizations by adopting them as reusable, abstract components. More dataset information can be found in Appendix F. Here are some typical examples.
\begin{itemize}
\item Logistic regression (GD)
\begin{align*}
    \min_x \sum_{i=1}^m \log\bigl(1+e^{-b_i a_i^\top x}\bigr) + \lambda \|x\|_2^2.
\end{align*}
\item Sparse recovery in signal processing (PGM)
\begin{align*}
    \min_x \frac{1}{2} \|Ax-b\|^2 + \mu \|x\|_1.
\end{align*}
\item Total variation denoising (ADMM)
\begin{align*}
    \min_{x,z} \frac{1}{2} \|x-y\|^2 + \|z\|_1, \quad \text{s.t.} \quad  Dx=z.
\end{align*}
\end{itemize}

\paragraph{Base Model} We use DeepSeek-R1 \cite{deepseekai2025} and DeepSeek-V3 \cite{deepseekai2025deepseekv3} as base models in our autoformalization pipeline. Unlike theorem proving specific LLMs, our structure-to-instance autoformalization task requires not only formal reasoning but also understanding informal language, identifying abstract structures, and generating Lean definitions and proofs. These models are chosen for their strong language understanding and multi-step instruction-following abilities. Model hyperparameters are listed in Appendix D.3.

% \begin{itemize}
%     \item Lasso problem
%     \item Wavelet decomposition model
%     \begin{align*}
%         \min_d \|\lambda \odot d\|_1 + \frac{1}{2}\|AW^Td-b\|^2
%     \end{align*}

%     \item Balanced wavelet model
%     \begin{equation*}
%         \min_{\alpha} \; \|\lambda \odot \alpha\|_1 
%         + \frac{\kappa}{2} \left\| (I - W W^\top) \alpha \right\|_2^2 
%         + \frac{1}{2} \left\| A W^\top \alpha - b \right\|_2^2
%     \end{equation*}

%     \item Joint sparse coding
%     \begin{align*}
%         \min_{x \in \mathbb{R}^{n},\; y \in \mathbb{R}^{m}} \frac{1}{2} \left\| A x + B y - b \right\|_2^2 + \lambda_1 \|x\|_1 + \lambda_2 \|y\|_1
%     \end{align*}
% \end{itemize}

\subsection{Results and Analysis}
With no former autoformalization work related to whole file generation, we propose the following four aspects to evaluate the output: (1) Definition (syntax-correct definitions without sorry); (2) Theorem (syntax-correct theorem statements, possibly with sorry); (3) Instance (syntax-correct instance declarations, possibly with sorry); (4) Full file (considering all four aspects above). To obtain the score, we interact with the Lean environment to collect the corresponding error messages, assigning each message to its associated definition, theorem, or instance. The success ratio is defined as the proportion of syntactically correct definitions, theorem statements, and instance declarations relative to their respective totals. Besides evaluation through type check, majority voting is also used to examine the semantic correctness. We follow the settings in \cite{liu2025rethinking} and use DeepSeek-V3 with temperature $T=0.7$ with 16 rounds\footnote{Detailed prompt templates are provided in Appendix D.2.}. The model is required to rate the output from the aspects of problem formalization, algorithm correctness, update scheme explicitness, theoretical analysis and proof completion rate. The score ranges from 0 to 100.

We compare the completion rate of SITA under two different base models with that of direct generation. Existing proof generation and autoformalization models are not designed to handle the structured formalization task addressed in our work, which involves generating definitions, instances, and theorems in a unified pipeline. To the best of our knowledge, there is currently no dedicated method capable of performing such structure-to-instance autoformalization. Therefore, we limit our comparison to general-purpose LLMs as a baseline for assessing the effectiveness of our approach. The results are shown in Table~\ref{table:completion}. All the results are under generation with 3 attempts and fix iteration with 3 iterations. A case study of both successful and failed cases generated by SITA is provided in Appendix G.1. More detailed results are given in Appendix H.

\begin{table}[htbp]
\centering
\begin{tabular}{lccccc}
\toprule
\textbf{Model} & \textbf{Def} & \textbf{Thm} & \textbf{Instance} & \textbf{File} & \textbf{MV} \\
\midrule
Direct-V3 & 27.9\% &  28.0\%& 22.8\%
 & 0.0\% & 50.2\\
Direct-R1 &62.8\% &  25.6\%&25.7\%& 0.0\%& 46.0 \\
SITA-V3 & 91.0\% & 86.7\%& 90.8\% & 27.2\% & 66.1 \\
SITA-R1 & 93.8\% &95.6\% & 95.4\% & 57.14\% & 76.9\\
\bottomrule
\end{tabular}
\caption{Formalization completion rate comparison.  \textbf{Direct-V3}: direct generation with DeepSeek-V3. \textbf{Direct-R1}: direct generation with DeepSeek-R1. \textbf{SITA-V3}: our framework using DeepSeek-V3.  \textbf{SITA-R1}: our framework using DeepSeek-R1. \textbf{Def}, \textbf{Thm}, \textbf{Instance} and \textbf{File} denotes four criteria for the evaluation of entire file generation. \textbf{MV}: results from majority voting.}
\label{table:completion}
\end{table}

As shown in Table~\ref{table:completion}, our proposed SITA framework significantly outperforms the direct generation baselines across all evaluation dimensions. Notably, SITA-R1 achieves an overall file-level success rate of 57.14\%. In contrast, direct generation methods, whether using DeepSeek-V3 or DeepSeek-R1, fail to produce any fully correct formalization at the file level, showing the substantial difficulty of end-to-end file level automatic formalization. One of the challenge lies in effectively leveraging newly synthesized concepts to instantiate appropriate type-class patterns and formulate theorems. Although current models are partly capable of generating definitions, they often fall short when integrating these definitions into semantically coherent lemmas or instances.

For those files fail to generate fully correct files, we examine the success rate of the statement generation. In addition to analyzing definitions and statements, we also study the generation of proofs in cases where the files are correctly generated. Proofs are evaluated only when the generation of definitions and statements succeeds. This is because Lean’s type-checking mechanism cannot reliably report errors in proofs if the definitions or statements contain type errors. As shown in Table~\ref{table: problem class}, we report two key metrics: the syntactic correctness rate of the definitions and statements in files that contain generation errors (SC), and the proof success rate in files that are generated without any syntax-level problems (PS). We report statistics including the number of definitions and theorems, and average file length for each class. 

\begin{table}[htbp]
\centering
\begin{tabular}{lccccc}
\toprule
\textbf{Class} & \textbf{SC} & \textbf{PS} & \textbf{DM}  & \textbf{TM} & \textbf{FL} \\
\midrule
GD         &  83.36\%  &   53.77\%  &   4    &  7    &     96 \\
PGM        &  98.1\%   &  62.96\%  &   6    &   9  &  116  \\
Nesterov   &  97.8\%    &   63.28\%       & 6    &   8    & 119     \\
BCD        &  88.82\%    &  50.55\%    &  9   & 17      & 194     \\
ADMM       &  85.02\%    &  20.00\%    & 6    & 8      &  121     \\
\midrule
\textbf{Overall} & 90.72\%  & 51.23\%   & 6.21 & 9.93   & 129.83     \\
\bottomrule
\end{tabular}
\caption{Syntactic correctness rate and proof completion rate. 
\textbf{SC}: syntactic correctness rate of statements in the failed cases; 
\textbf{PS}: proof completion rate in the success cases;
\textbf{DM}: total number of definitions;
\textbf{TM}: total number of lemmas and theorems; 
\textbf{FL}: average file length (lines).}
\label{table: problem class}
\end{table}

From the results, we observe that even among the files that are not fully correct, a significant portion of their component are still syntactically valid. 

This suggests that most of the errors are concentrated in a small subset of complex or ambiguous definitions or statements, while the majority of the content is already valid and usable. 
Therefore, with SITA as the backbone, the cost of human intervention can be reduced. Users only need to focus on refining a few critical pieces rather than writing the entire file from scratch.

On the other hand, for files that are syntactically correct, the proof success rate exceeds 50\% in the simpler classes. This indicates that the model is capable of completing a substantial portion of proofs autonomously, especially for elementary or structurally well-defined lemmas. The remaining failures are mostly associated with intricate reasoning steps or subtle dependencies, which are inherently harder for current models to resolve without fine-grained guidance.  

\paragraph{Generated Formal Problems}
From this structure to instance autoformalization procedure, we generate 88 correctly complied files and 449 theorems and 222 of them are with correct proofs. Most of these problems concentrates on the properties of concrete functions, such as the convexity, the Lipschitz continuity, or the KL property. We reorganize the file and obtain a benchmark Opt-bench consisting of formalized problems focus on analysis and optimization.

\subsection{Ablation Study}
We conduct ablation studies to understand each component’s contribution in the SITA framework. More concrete results can be found in Appendix H.
\paragraph{Procedure Ablation Study} We study the effect of key modules in the SITA pipeline. As shown in Table~\ref{table:ablation-pipeline}, removing the linking error recovery module leads to only a moderate drop in both syntactic correctness and proof success, suggesting that the model possesses a certain capacity to adapt based on Lean’s error messages. Retrieval-augmented error correction contributes to improves the stability and precision of error correction. Furthermore, disabling the proof refinement stage is associated with a decrease in performance, indicating that this component plays a role in improving partially correct roofs. The increase in file length also validates the effectiveness of proof refinement. In addition, removing in-class examples from the prompt (replaced with generic examples) leads to a decline in all metrics, demonstrating that in-context alignment plays a vital role in guiding the generation toward syntactically correct and semantically aligned outputs.

\begin{table}[htbp]
\centering
\begin{tabular}{lcccc}
\toprule
\textbf{Configuration}& \textbf{FS (\%)} & \textbf{SC (\%)} & \textbf{PS (\%)} & \textbf{FL} \\
\midrule
w/o example & 9.5\% &  74.1 \% &  42.5\% & 100\\
w/o link errs & 45.8 \%&  83.4 \%&  49.8 \%  & 131\\
w/o proof refine &  57.14\% &90.7\%& 23.45\% & 119 \\
Full pipeline &  57.14\%& 90.7\% & 51.23\%   & 134   \\
\bottomrule
\end{tabular}
\caption{Ablation study of pipeline stages.
\textbf{FS}: File success compilation rate;
% \textbf{SC}: syntactic correctness rate of definitions and statements in the failed cases; 
% \textbf{PS}: proof completion rate (\%) in the success cases;
\textbf{w/o example}: pipeline without an exemplar within the same category in the prompt, uses general examples instead. 
\textbf{w/o link errs}: the pipeline without correction tips retrieval from error knowledge base;
\textbf{w/o proof refine}: the pipeline without secondary proof generation.}
\label{table:ablation-pipeline}
\end{table}

\begin{figure}[htbp]
    \centering
    \includegraphics[width=0.7\linewidth, height = 0.38\linewidth]{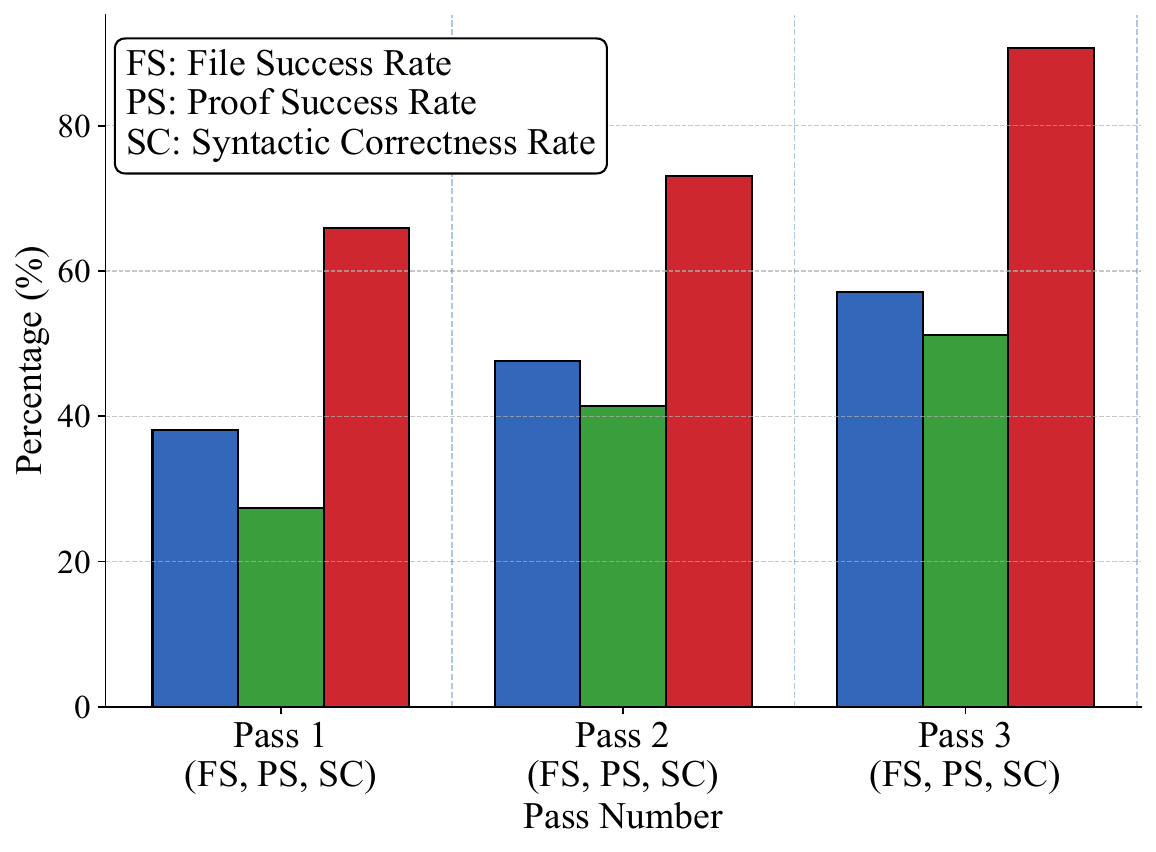}
    \caption{Evaluation performance across generation passes.}
    \label{fig:pass}
\end{figure}

\paragraph{Time Consumption of Each Part} 
We also compare the time consumption of each part. Correction steps dominate the runtime: the corrections for the backbone and the corrections for the proofs account for 56.5 \% and 22.5 \%. While the remaining three stages, harmless fixing, generation, and proof generation, each consume roughly the same, comparatively minor share of the total time. 

\paragraph{Generation Pass Ablation Study}
We study the impact of the number of pass rounds on three evaluation metrics: file success rate, SC rate and the PS rate. As shown in Figure~\ref{fig:pass}, all three metrics improve with the number of passes. This indicates that the number of passes can influence the overall quality and correctness of the generated outputs.

\section{Conclusion}
\label{sec:conclusion}

We propose SITA, a novel pipeline for structure-to-instance theorem autoformalization, introducing a new paradigm of structured mathematical reasoning for research level mathematics. By aligning abstract formal structures with concrete problem instances, our pipeline enables scalable, modular proof reuse. Experiments on diverse optimization problems show its effectiveness in generating correct formal definitions, instance declarations, and verified proofs. This framework paves the way for reusable, verified libraries and structured datasets to support future advancements.

\section*{Acknowledgments}
Z. Wen was supported in part by National Key Research and Development Program of China under the grant number 2024YFA1012903, the National Natural Science Foundation of China under the grant numbers 12331010 and 12288101, and the Natural Science Foundation of Beijing, China under the grant number Z230002.

\bibliography{aaai2026}

\newpage
\appendix

\section{Mathematical Formalization Using Lean}
Lean is an interactive theorem prover and programming language based on dependent type theory, developed to support formalized mathematics for mathematicians. Lean’s core library, mathlib, % \cite{mathlibcommunity}
offers an extensive and growing repository of formalized mathematics. It includes over twenty thousands of definitions and theorems across a wide range of mathematical domains, designed with reusability and formal rigor. Lean follows the propositions-as-types principles, where propositions are defined as types and proofs as terms inhabiting those types. To prove a theorem in Lean, one can construct a term of the corresponding type using previously proven theorems and definitions.

Another a widely used proving style in the Lean theorem prover is the tactic proof. Unlike term-style proofs, which construct a proof object directly, tactic proofs proceed by incrementally transforming the current state using a sequence of commands called tactics. Each tactic operates on the current proof state, which comprises the local context of assumptions and the goal to be proven. Tactics include introducing variables (intro), destructuring assumptions (rcases), applying known results (apply, exact), or simplifying expressions (simp, rw). This step-by-step manipulation continues until all subgoals are resolved, typically by matching them with assumptions or reducing them to trivially true statements. Tactic proofs closely mirror informal mathematical reasoning and offer both clarity and flexibility during formal development. This proving style plays a central role in our formal verification process.

In Lean, type classes provide a powerful abstraction mechanism that supports modular, reusable, and structured formalizations. Type classes are used to encode algebraic structures (such as groups or vector spaces). They are also equally effective for organizing algorithmic frameworks, such as families of first-order optimization methods. A type class in Lean can be viewed as a named collection of definitions and assumptions, including both data (e.g., parameters or functions) and properties (e.g., positivity, well-posedness). This allows for concise theorem statements and enables automatic instance inference, where Lean can automatically fill in relevant structures based on context.

An instance is a declaration that a particular object or structure satisfies the requirements of a given type class in Lean. While a type class specifies a collection of data and properties, effectively describing a certain interface or set of assumptions, an instance provides a concrete implementation that fulfills those requirements. In the context of formalizing algorithms, instances allow users to encode the data and assumptions specific to a problem in a structured and reusable way. General theorems can then be applied to any concrete instance that satisfies the necessary type class, enabling a high degree of automation and abstraction in proofs.

Optlib is a Lean 4 library for formalizing key concepts and results in mathematical optimization. It includes definitions and formal proofs related to convex functions, subgradients, first-order optimality conditions, Lagrangian duality, tangent cones, and KKT conditions. The library also contains formal convergence proofs for classic algorithms such as gradient descent and block coordinate descent. Optlib provides a structured foundation for rigorous reasoning about optimization theory within the Lean proof assistant. Our experiments mainly builds on this repository.

\section{Mathematical Structures with Algorithm in Different Mathematical Realms}
There are also many examples of mathematical structures equipped with operations beyond the domain of optimization methods. Here is an example of abstract algebra.
\begin{example} 
Euclidean division on a Euclidean domain is given as:
\(
  \mathcal{S}_{\mathrm{ED}} = \bigl\langle
    \mathcal{D},\,\mathcal{O},\,\mathcal{C},\,\mathcal{T}
  \bigr\rangle.
\)
\begin{itemize}
  \item \(\mathcal{D}\):  
  A Euclidean domain \( R \), equipped with a Euclidean valuation  
  \( \delta: R \setminus \{0\} \to \mathbb{N} \).
  \item \(\mathcal{O}\): The division on the Euclidean domain is given as follows. For any $a, b \in R$ with $b \ne 0$, find $q, r \in R$ such that $a = bq + r$, with either $r = 0$ or $\delta(r) < \delta(b)$.
  \item \(\mathcal{C}\):  
  The function \( \delta \) satisfies the Euclidean property, ensuring that such a division is always possible. 
  \item \(\mathcal{T}\): Theorems include:
  \begin{itemize}
    \item Existence and uniqueness of \( q, r \);
    \item Termination of the Euclidean algorithm;
    \item Properties of the greatest common divisor;
  \end{itemize}
\end{itemize}
\end{example}
Euclidean division can be viewed as an operation defined on a Euclidean domain. This framework admits concrete applications in a variety of algebraic structures, such as the natural numbers, Gaussian integers, and univariate polynomial rings over a field, each equipped with their respective division operations. There are also examples in graph theory.

\begin{example}
Graph traversal on a finite graph is defined as:
\begin{itemize}
  \item \(\mathcal{D}\):  A finite graph \( G = (V, E) \), where $V$  denotes the set of vertices and and $E$ the set of edges. 
  \item \(\mathcal{O}\):  
  A traversal algorithm starting from certain node. The algorithm explores vertices in layers.

  \item \(\mathcal{C}\):  
  The graph and the traversal procedure satisfies:
  \begin{itemize}
    \item The graph is undirected and connected;
    \item Every node is visited at most once;
    \item All nodes reachable from the source are finally visited.
  \end{itemize}

  \item \(\mathcal{T}\): Theorems include the time complexity analysis of the certain algorithms.
\end{itemize}
\end{example}
The finite graph can be instantiated as a specific concrete graph, and the corresponding traversal algorithm can be chosen to suit the context, such as breadth-first search or depth-first search, depending on the desired exploration strategy.

Such mathematical structures are widely used across different areas of mathematics, as they provide a systematic way to define domains, specify operations, impose constraints, and derive key properties within a unified and reusable framework.

\section{Illustrative Examples for SITA Problem}
\subsection{PGM Applied to the Lasso Problem}
The first part of the code represents the formalization of the abstract structure. The proof of the abstract structure is obtained from the library Optlib. The general formal definition of composite optimization problem and proximal gradient method is in abstract form and defined on Hilbert spaces, which is a generalization of $\mathbb{R}^n$.
\begin{lstlisting}
variable {E : Type*} [NormedAddCommGroup E] [InnerProductSpace ℝ E] [CompleteSpace E] [FiniteDimensional ℝ E]
class composite_pro (f : E → ℝ) (h : E → ℝ)
def composite_pro.target (_ : composite_pro f h) := f + h
class pg (pro : composite_pro f h) (x₀ : E) :=
  t : ℝ
  x : ℕ → E
  update : ∀ k : ℕ, prox_prop (t • h) (x k - t • (gradient f) (x k)) (x (k + 1))
  initial : x 0 = x₀
\end{lstlisting}
Given appropriate assumptions $\mathcal{C}$ (provided in the parentheses), we obtain the formalized convergence theorem $\mathcal{T}$.
\begin{lstlisting}
theorem pg_converge
  (xm : E) (L : NNReal)
  (fconv : ConvexOn ℝ univ f) (hconv : ConvexOn ℝ univ h)
  (h₁ : Differentiable ℝ f) (h₂ : LipschitzWith L (gradient f))
  (tpos : 0 < alg.t) (step : alg.t ≤ 1 / L) (hL : L > (0 : ℝ)) :
  ∀ (k : ℕ+), (pro.target (alg.x k) - pro.target xm)
    ≤ 1 / (2 * k * alg.t) * ‖x₀ - xm ‖ ^ 2 := by sorry
\end{lstlisting}

This completes a full formalized example of the abstract structure $\mathcal{S}$. In what follows, the formalized code serves as the expected output for LLM generation. The desired instance formalization $\mathbb{I}$ for the Lasso problem includes defining the definitions $\mathcal{D}$ and operation $\mathcal{O}$ as follows:
\begin{lstlisting}
class Lasso_pro {m n : ℕ} (A : Matrix (Fin m) (Fin n) ℝ) (b : (Fin m) → ℝ) (mu : ℝ) where
  (hA : A ≠ 0)
  (hmu : mu > 0)

variable {A : Matrix (Fin m) (Fin n) ℝ} {b : Fin m → ℝ} {mu : ℝ}

def Lasso_pro.f (_ : Lasso_pro A b mu) : EuclideanSpace ℝ (Fin n) → ℝ :=
  fun x ↦ 1 / 2 * ‖ A *ᵥ x - b ‖₂ ^ 2

def Lasso_pro.g (_ : Lasso_pro A b mu) : EuclideanSpace ℝ (Fin n) → ℝ :=
  fun x ↦ mu * ‖ x ‖₁

def Lasso_pro.target (self : Lasso_pro A b mu) : EuclideanSpace ℝ (Fin n) → ℝ :=
  fun x ↦ self.f x + self.g x
  
class pg_Lasso (pro : Lasso_pro A b mu) (x₀ : EuclideanSpace ℝ (Fin n)) where
  t : ℝ
  x : ℕ → EuclideanSpace ℝ (Fin n)
  y : ℕ → EuclideanSpace ℝ (Fin n)
  ht : t > 0
  update1 : ∀ k : ℕ,
    let grad : EuclideanSpace ℝ (Fin n) := Aᵀ *ᵥ (A *ᵥ x k - b)
    y k = x k - t • grad
  update2 : ∀ (k : ℕ), ∀ i, x (k + 1) i = (Real.sign (y k i) * (max (abs (y k i) - t * mu) 0))
  initial : x 0 = x₀
\end{lstlisting}
Note that the update form for proximal gradient method on Lasso problem is explicitly written. The proximal operator of the $\ell_1$ norm and the gradient is calculated. Hence these definitions is defined separately from the general definitions. The key to linking the abstract structure with the concrete example lies in the following instances. These instance theorems assert that the concrete definitions and operations are consistent with their abstract counterparts by mapping bundled variables and proving required properties.
\begin{lstlisting}
instance Lasso_pro.composite_pro (self : Lasso_pro A b mu) :
    composite_pro self.f self.g where

instance pg_Lasso.pg (self : pg_Lasso pro x₀) : pg (Lasso_pro.composite_pro pro) x₀ where
  t := self.t
  x := self.x
  initial := self.initial
  update := by sorry
\end{lstlisting}
In the \texttt{update} part, we prove that the update scheme defined in \texttt{pg\_Lasso} coincides with the general definitions in \texttt{pg} class. To prove the final convergence theorem, we need to first formalize some theorems based on the conditions in the theorem \texttt{pg\_converge}. Some of these are given as below.
\begin{lstlisting}
lemma Lasso_problem.ConvexOn_f (self : Lasso_problem A b mu) :
    ConvexOn ℝ Set.univ self.f  := by
\end{lstlisting}
\begin{lstlisting}
lemma Lasso_problem.lip_f (self : Lasso_problem A b mu) :
    LipschitzWith self.l (gradient self.f) := by
\end{lstlisting}
Other conditions needed are given in lemmas named \texttt{pro.ConvexOn\_g}, \texttt{pro.diff\_f}. 
By applying the instances and the verified assumptions in the abstract theorem, we obtain the corresponding convergence result for the Lasso problem. 
\begin{lstlisting}
theorem Lasso_convergence (alg : pg_Lasso pro x₀)
    (xm : EuclideanSpace ℝ (Fin n))
    (ht2 : alg.t ≤ 1 / pro.l):
    ∀ (k : ℕ+), (pro.target (alg.x k) - pro.target xm)
      ≤ 1 / (2 * k * alg.t) * ‖ x₀ - xm ‖ ^ 2 := by
intro k
  apply proximal_gradient_converge (alg := alg.proximal_gradient_method)
    xm pro.l pro.ConvexOn_f pro.ConvexOn_g pro.diff_f pro.lip_f alg.ht ht2 pro.lpos k
\end{lstlisting}
We can see that the convergence of the proximal gradient applied on Lasso problem can be simply proved from the tactics above, given some properties about the target functions proved (e.g. in \texttt{pro.ConvexOn\_f} and \texttt{pro.lip\_f}). 

This simple example of the autoformalization task of SITA illustrates the core concept, which remains consistent even in more complex real-world scenarios. In these scenarios, the target function and update scheme become more intricate. The fundamental idea behind autoformalization is to leverage the information and structure present in the template file to guide the reasoning process and facilitate the generation of structure-to-instance theorems.

\subsection{Lean Formalization with its Corresponding Natural Language Back-translation}

Based on SITA, we aim to generate a complete natural language report for a given optimization problem via structure-to-instance formalization. This extends the utility of autoformalization by bridging formal code with natural language. The following presents a fully formalized generation of a wavelet decomposition model, which will be used for reverse translation and evaluation.

\begin{lstlisting}
open Set Real Matrix Finset Filter Bornology BigOperators Topology Classical

noncomputable section WAVELET

local notation "‖" x "‖₂" => @Norm.norm _ (PiLp.instNorm 2 fun _ ↦ ℝ) x
local notation "‖" x "‖₁" => @Norm.norm _ (PiLp.instNorm 1 fun _ ↦ ℝ) x
local notation "|‖" A "|‖" => ‖(Matrix.toEuclideanLin ≪≫ₗ LinearMap.toContinuousLinearMap) A‖₊

variable {m n : ℕ} {M : Matrix (Fin m) (Fin n) ℝ} {b : Fin m → ℝ} {lam : Fin n → ℝ}

class Wavelet_model (M : Matrix (Fin m) (Fin n) ℝ) (b : Fin m → ℝ) (lam : Fin n → ℝ) where
  hlam : ∀ i, lam i ≥ 0

def Wavelet_model.f (pro : Wavelet_model M b lam) : EuclideanSpace ℝ (Fin n) → ℝ :=
  fun d ↦ 1 / 2 * ‖ M *ᵥ d - b ‖₂ ^ 2

def Wavelet_model.g (pro : Wavelet_model M b lam) : EuclideanSpace ℝ (Fin n) → ℝ :=
  fun d ↦ ‖ fun i ↦ lam i * d i ‖₁

def Wavelet_model.target (pro : Wavelet_model M b lam) : EuclideanSpace ℝ (Fin n) → ℝ :=
  pro.f + pro.g

def Wavelet_model.l (pro : Wavelet_model M b lam) : NNReal := |‖Mᵀ * M|‖

instance Wavelet_model.composite_problem (pro : Wavelet_model M b lam) :
    composite_problem pro.f pro.g where

class Nesterov_wavelet (pro : Wavelet_model M b lam) (x₀ : EuclideanSpace ℝ (Fin n)) where
  hl : pro.l > (0 : ℝ)
  x : ℕ → EuclideanSpace ℝ (Fin n)
  y : ℕ → EuclideanSpace ℝ (Fin n)
  w : ℕ → EuclideanSpace ℝ (Fin n)
  t : ℕ → ℝ
  γ : ℕ → ℝ
  oriy : y 0 = x 0
  initial : x 0 = x₀
  teq : ∀ n : ℕ, t n = 1 / pro.l
  γeq : ∀ n : ℕ, γ n = 2 / (2 + n)
  update1 : ∀ k : ℕ+, y k = x k + (γ k * (1 - γ (k - 1)) / (γ (k - 1))) • (x k - x (k - 1))
  update2 : ∀ k : ℕ,
    let grad : EuclideanSpace ℝ (Fin n) := Mᵀ *ᵥ (M *ᵥ y k - b)
    w k = y k - t k • grad
  update3 : ∀ k : ℕ, x (k + 1) =
    fun i ↦ Real.sign (w k i) * (max (|w k i| - t k * lam i) 0)

variable {pro : Wavelet_model M b lam} {x₀ : EuclideanSpace ℝ (Fin n)}

lemma Wavelet_model.hasGradient (pro : Wavelet_model M b lam) :
    ∀ d, HasGradientAt pro.f (Mᵀ *ᵥ (M *ᵥ d - b)) d := by
  apply affine_sq_gradient

lemma Wavelet_model.gradient_f (pro : Wavelet_model M b lam) :
    ∀ d, gradient pro.f d = Mᵀ *ᵥ (M *ᵥ d - b) := by
  exact fun d ↦ HasGradientAt.gradient (pro.hasGradient d)

lemma Nesterov_wavelet.update_cor (self : Nesterov_wavelet pro x₀) :
    ∀ (k : ℕ), prox_prop (self.t k • pro.g) (self.y k - self.t k • gradient pro.f (self.y k)) (self.x (k + 1)) := by
  sorry

instance 
Nesterov_wavelet.Nesterov_first_fix_stepsize (self : Nesterov_wavelet pro x₀) :
    Nesterov_first_fix_stepsize (Wavelet_model.composite_problem pro) x₀ where
  hl := self.hl
  x := self.x
  y := self.y
  t := self.t
  γ := self.γ
  oriy := self.oriy
  initial := self.initial
  teq := self.teq
  γeq := self.γeq
  update1 := self.update1
  update2 := self.update_cor

lemma Wavelet_model.ConvexOn_f (pro : Wavelet_model M b lam) :
    ConvexOn ℝ univ pro.f := by
  unfold Wavelet_model.f
  exact affine_sq_convex M b

lemma Wavelet_model.ConvexOn_g (pro : Wavelet_model M b lam) :
    ConvexOn ℝ univ pro.g := by
  sorry

lemma Wavelet_model.diff_f (pro : Wavelet_model M b lam) :
    Differentiable ℝ pro.f := by
  exact fun x ↦ HasGradientAt.differentiableAt (pro.hasGradient x)

lemma Wavelet_model.lip_f (pro : Wavelet_model M b lam) :
    LipschitzWith pro.l (gradient pro.f) := by
  rw [lipschitzWith_iff_norm_sub_le]; intro x y
  rw [pro.gradient_f, pro.gradient_f]
  rw [← Matrix.mulVec_sub, ← sub_add, sub_add_eq_add_sub, sub_add_cancel]
  rw [← Matrix.mulVec_sub]
  simp
  apply Matrix.l2_opNorm_mulVec (Mᵀ * M)

theorem wavelet_convergence (alg : Nesterov_wavelet pro x₀)
    (xm : EuclideanSpace ℝ (Fin n)) (minφ : IsMinOn pro.target univ xm) :
    ∀ (k : ℕ), pro.f (alg.x (k + 1)) + pro.g (alg.x (k + 1)) - pro.f xm - pro.g xm ≤
    2 * pro.l / (k + 2) ^ 2 * ‖x₀ - xm‖ ^ 2 := by
  apply Nesterov_converge (alg := alg.Nesterov_first_fix_stepsize)
    pro.l alg.hl pro.diff_f pro.ConvexOn_f pro.lip_f pro.ConvexOn_g minφ

end WAVELET
\end{lstlisting}

\paragraph{Back-translation}
We employ DeepSeek-R1 to back-translate the above Lean formalization into a LaTeX-formatted technical report describing Nesterov's first method applied to the wavelet decomposition model. The prompt explicitly instructs the model to preserve all definitions, lemmas, theorems, and proof structures; translate tactic-based proofs into natural language reasoning; and format the output using standard LaTeX environments without introducing custom commands. The following section was fully generated by the model based on these constraints.

\begin{definition}
A \textbf{Wavelet decomposition model} is defined for a matrix $M \in \mathbb{R}^{m \times n}$, a vector $b \in \mathbb{R}^m$, and coefficients $\lambda \in \mathbb{R}^n$ satisfying $\lambda_i \geq 0$ for all $i$. The components are:
\begin{align*}
f(d) &:= \frac{1}{2} \|M d - b\|_2^2 \\
g(d) &:= \|\lambda \odot d\|_1 \\
\varphi(d) &:= f(d) + g(d) \\
L &:= \|M^\top M\|_{op}
\end{align*}
where $\odot$ denotes componentwise multiplication and $\|\cdot\|_{op}$ is the operator norm.
\end{definition}

\begin{definition}
The \textbf{Nesterov first fixed-stepsize algorithm} for wavelet decomposition initializes at $x_0$ and iterates:
\begin{align*}
t_k &:= 1/L \\
\gamma_k &:= 2/(2+k) \\
y_k &:= x_k + \frac{\gamma_k (1 - \gamma_{k-1})}{\gamma_{k-1}} (x_k - x_{k-1}) \quad (k \geq 1) \\
w_k &:= y_k - t_k (M^\top (M y_k - b)) \\
x_{k+1}(i) &:= \mathrm{sign}(w_k(i)) \cdot \max(|w_k(i)| - t_k \lambda_i, 0)
\end{align*}
with initial conditions $y_0 = x_0$ and $\gamma_{-1}$ unused.
\end{definition}

\begin{lemma}
The function $f$ has gradient $\nabla f(d) = M^\top (M d - b)$ at every $d$.
\end{lemma}
\begin{proof}
By application of the affine square gradient lemma.
\end{proof}

\begin{lemma}
The proximal update step satisfies:
\[
x_{k+1} = \mathrm{prox}_{t_k g}(y_k - t_k \nabla f(y_k))
\]
\end{lemma}
\begin{proof}
Proof incomplete: Proximal operator characterization for $\ell_1$-norm.
\end{proof}

\begin{lemma}
The function $f$ is convex on $\mathbb{R}^n$.
\end{lemma}
\begin{proof}
By application of the affine square convexity lemma to $f(d) = \frac{1}{2}\|Md - b\|_2^2$.
\end{proof}

\begin{lemma}
The function $g$ is convex on $\mathbb{R}^n$.
\end{lemma}
\begin{proof}
Proof incomplete: Convexity of weighted $\ell_1$-norm.
\end{proof}

\begin{lemma}
The function $f$ is differentiable everywhere.
\end{lemma}
\begin{proof}
Direct application of gradient existence.
\end{proof}

\begin{lemma}
The gradient $\nabla f$ is $L$-Lipschitz continuous.
\end{lemma}
\begin{proof}
Rewriting using gradient expressions and matrix norm properties:
\[
\|\nabla f(x) - \nabla f(y)\| = \|M^\top M(x - y)\| \leq \|M^\top M\|_{op} \|x - y\|
\]
by definition of the operator norm.
\end{proof}

\begin{theorem}
Let $x^*$ minimize $\varphi = f + g$. The Nesterov algorithm satisfies:
\[
\varphi(x_{k+1}) - \varphi(x^*) \leq \frac{2L \|x_0 - x^*\|^2}{(k+2)^2}
\]
\end{theorem}
\begin{proof}
Direct application of the generic Nesterov convergence theorem to the composite problem, using the established properties: differentiability of $f$, convexity of $f$ and $g$, $L$-Lipschitz gradient for $f$, and valid proximal updates.
\end{proof}

\paragraph{Evaluation}
The back-translated output serves not only as a test of linguistic fidelity but more importantly as an evaluation of the system’s ability to generate high-quality, structurally complete mathematical content from a compact problem description. Given the minimal input, consisting of a problem name, a mathematical objective, and an algorithm class, the resulting formalization demonstrates a coherent and rigorous unfolding of the wavelet decomposition problem and its solution via Nesterov’s first-order method. While the current formalization proof mainly relies on relatively advanced theorems in Optlib or Mathlib, which results in a concise natural-language explanation. We can further expand the underlying proof details, with the help of LLM, into a more detailed and comprehensive article.

Specifically, the generated content includes well-formed definitions of the objective components (the smooth term $f$, the non-smooth term $g$, and their sum $\varphi$), appropriate algorithmic details such as initialization and iteration schemes, and a convergence theorem with supporting lemmas. The definitions are mathematically accurate, and the key assumptions, such as convexity, differentiability, and Lipschitz continuity, are properly identified and supported by lemmas. Proofs are partially completed where information suffices, while incomplete elements are clearly marked, maintaining formal integrity.

Overall, the output exhibits strong structural coherence and conceptual correctness. Despite the lack of detailed input, the system successfully reconstructs a nontrivial optimization setting with appropriate theoretical guarantees. This highlights the potential of structure-to-instance autoformalization not only for formal encoding but also for the autonomous generation of interpretable and pedagogically useful mathematical documentation. This approach demonstrates the feasibility of using formal systems as a foundation for generating natural language explanations that are both accurate and logically grounded.

\section{Algorithm Details}
\subsection{Pseudocode of the Algorithm}
Algorithm~\ref{alg:all} summarizes the main pipeline of SITA. Starting from a structured formalization template and an instance described in natural language, the system proceeds through a series of modular stages: generating instance-specific definitions and statements, type-checking and correcting errors, and finally proving all pending goals. Each phase leverages LLMs for generation, while interacting tightly with Lean’s type checker to ensure correctness. The algorithm iteratively repairs and validates the proof state until a complete and verifiable Lean file is produced. The \textbf{automatic corrector} is defined as in Algorithm \ref{alg:fix}.
\begin{algorithm}[htbp]
\caption{Structure‐to‐Instance Autoformalization}
\label{alg:all}
\begin{algorithmic}[1]
\Require $\mathcal{S}=\langle\mathcal{D},\mathcal{O},\mathcal{C},\mathcal{T}\rangle$ (formalized template), instance information $I$ (natural language problem)
\Ensure Lean file $F$ defining and proving the instance $\mathbb{I}$ with $\langle\mathbb{D},\mathbb{O},\mathbb{C},\mathbb{T}\rangle$.

\State Base on the template, generate $\mathbb{D}, \mathbb{O}, \mathbb{T}$(with sorry) from $ \textsc{LLM}(\mathcal{D},\mathcal{O},\mathcal{T}|I)$  
  \quad \Comment{Definition and statement generation }
\State \textbf{Lean.Check}($\mathbb{D}, \mathbb{O}, \mathbb{T}$)  
  \Comment{Type‐check}

\State \textbf{Automatic Corrector} ($\mathbb{D}, \mathbb{O}, \mathbb{T}$) \Comment{Error fix} 
\State Find all sorry needed to prove as $\Pi$
\ForAll{goal $g\in \Pi$}
  \Repeat
    \State $p \gets \textsc{LLM}(\text{proof}|g)$ \Comment{Whole proof generation}
    \State $e \gets \textbf{Lean.Check}(p)$
    \If{$e\neq \texttt{OK}$}
      \State \textbf{Automatic Corrector} (p) \Comment{Error fix}
    \EndIf
  \Until{$e=\texttt{OK}$ or maximum step reached}
  
\EndFor
\State Fix all the rest errors and get the entire Lean file $F$.
\State \textbf{Lean.Check}($F$)  
  \Comment{Final verification}
\State \Return $F$
\end{algorithmic}
\end{algorithm}

\begin{algorithm}
\caption{Error Fix Framework}
\label{alg:fix}
\begin{algorithmic}[1]
\State \textbf{Input:} Initial Lean code with errors $C$, Error knowledge base $\mathcal{K}$.
\State \textbf{Output:} Corrected Lean code $C'$.

\State \textbf{Initialize:} 
\State $C' \gets C$, $e \gets \text{True}$
\While{$e$ and maximum step not reached}
    \State Static rule-based fixes: $C' \gets \textbf{StaticFix}(C')$
    \State $e \gets \textbf{Lean.Check}(C')$
    \If{$error = \text{False}$}
        \State \textbf{Break} \Comment{Static method fix the code.}
    \EndIf
    
    \State Retrieve error from knowledge base 
    
    $\mathcal{K}(e) \gets \text{RetrieveFeedback}(\mathcal{K}, e)$
    \State $prompt \gets \text{ConstructPrompt}(e, \mathcal{K}(e))$
    
    \State Generate correction $C_1 \gets \text{LLMFix} (C', prompt)$
    \State \textbf{Recheck:}
    \State $e \gets \text{CheckForErrors}(C_1)$
    
    \If{$e = \text{False}$}
        \State Update the knowledge base: 
        
        $\quad \mathcal{K} \gets \text{GetKnowledge}(\mathcal{K}, e, C', C_1)$
        \State $C' \gets C_1$
        \State \textbf{Break} \Comment{Error resolved, stop iteration}
    \EndIf 
    \State $C' \gets C_1$
\EndWhile

\State \textbf{Return:} Corrected code $C'$
\end{algorithmic}
\end{algorithm}
\subsection{Prompts Used in the Algorithms}
\paragraph{Prompts for first generation of the backbone}

As a mathematical formalization expert and Lean 4 programming expert, you possess extensive experience and deep understanding in formalizing optimization problems and are proficient in Lean 4 programming. You are capable of defining new classes, definitions, and instances in Lean 4 and can derive theorems for specific problems based on the general structure.

You need to generate a complete Lean 4 formalization for a specific optimization problem instance. We already has a general formalization structure for the optimization method and the class of problems it applies to, and requires the creation of a formalization for a specific problem instance.

Your task is to generate a complete Lean 4 formalization of this specific optimization problem instance, strictly following the structure and the style of the provided structure reference Lean 4 file. You need to define new classes for the optimization problems and methods, define suitable definitions based on the classes. Besides, you need to link the formalization of the specific problem to the structure reference using instance in Lean4. Finally, you need to state theorems specialized from the structure reference under the setting of the concrete problems, and try to prove it based on the structure reference. If you cannot prove it, just write "sorry" in the proof.

The definition of the problem class should contain all the needed variables in ([name] : [Type]) format, no matter whether they are used in the properties here or not. Do not use ``variable" block with explicit definitions. If you use ``let" to define things, please give the corresponding type explicitly. You may need to use ``let" to define some intermediate variables. Please note that the output of matrix vector multiplication is defined using type of ``Fin n $\to$ $\mathbb{R}$". You may need to use ``let" to give the type as EuclideanSpace $\mathbb{R}$ (Fin n).

Requirements:

1. Strictly follow the structure of the reference file:

   - Problem definition (variables, objective function)
   
   - Algorithm implementation (parameters, iteration format)
   
   - Convergence theorem (statement only, proofs as ``sorry")
   
2. **Replace All proofs with ``sorry"**

3. Preserve all mathematical notation, naming conventions, and code style from the reference.

4. Ensure the generated code is syntactically correct Lean 4 code. Do not use functions not defined in Lean4.

5. An example of a Lean 4 formalization of abstract method applied to the concrete problem is provided. Your output must imitate its structure.

6. You should not add unneeded assumptions for the theorems.

7. Use same imports and namespaces as reference. Do not change the imports. You should not need to repeat the template.

Problem description: \{problem\}

Structure Reference Lean4 code: \{lean\_structure\}

Example Lean4 code: \{lean\_example\}

Output ONLY the complete Lean4 code WITHOUT any explanations.

\paragraph{Prompts for proof generation} The prompt used for proof generation is provided as below.

As a mathematical formalization expert and Lean 4 programming expert, you possess extensive experience and deep understanding in formalizing optimization problems and are proficient in Lean 4 programming. You are proving lemmas and theorems in Lean4.

    **CRITICAL TASK**: You MUST replace ALL ``sorry" placeholders with actual mathematical proofs. In Lean 4, ``sorry" is a placeholder that should be replaced with real proofs. Your job is to provide complete, rigorous proofs for each theorem and lemma.

    **IMPORTANT**: Do NOT output ``sorry" in your response. Every ``sorry" you see must be replaced with a proper proof. If you cannot complete a proof, use tactics like ``simp", ``rfl" ``assumption", ``exact", or provide step-by-step proof tactics.

    Requirements:
    
    1. **MANDATORY**: Replace every single ``sorry" with actual proof tactics or proof terms.
    
    2. Keep all other parts of the file unchanged (imports, definitions, theorem statements).
    
    3. Use the reference file to understand proof patterns and tactics.
    
    4. Wrap your complete code in ```lean4 ``` block.

    The file you need to prove (REPLACE ALL ``sorry" WITH REAL PROOFS) \{lean\_content\}

    Structure Reference Lean4 code (for proof patterns and tactics):
    \{example\_content\}

    Remember: Your output must have ZERO ``sorry" statements. Every theorem must have a complete proof!
    
    Output ONLY the complete Lean4 code WITHOUT any explanations.

\paragraph{Prompts for error correction} The prompt used for error correction is given below. The part ``theorem details" is added when the error involves applying theorems.

[Task] As an expert proficient in Lean, your task is to fix the Lean code with the error information Lean offers. Given the following code and list of compiler errors, return a fully fixed version. For definitions, please carefully fix the errors in the code, and for theorems, you can add ``sorry" to fix the code if you cannot prove it.
        
[Full Current Code] \{self.current\_code\}

[Error \{Number\}]

File: \{self.lean\_file\}

Line: \{error['line']\}

Error: \{error['message']\}

[Context] \{local context\}

[Full Block Context] \{error['full\_context']\}

[Top 3 Similar Error Solutions]

[Theorem Details]

[Fix Requirements]

1. Fix all above errors, output complete Lean4 code

2. Return the entire file content, not just fixes.

3. Wrap complete code in ```lean '''

4. Don't fix errors individually, provide a unified solution file.

5. Output ONLY the complete Lean4 code WITHOUT any explanations.

\paragraph{Prompts for fix explanation} The prompt used for fix explanation is given as follows.

As a Lean 4 expert, analyze these code changes and explain the fix professionally:

Original Code (with error): \{ original code\}

Fixed Code: \{ fix code\}

Please provide a concise but detailed explanation that:

Identifies the root cause of the error

Explains what specifically was changed

Describes why the fix works

Uses appropriate Lean terminology

Is under 200 words

Format your response as:

Error Type: \textless type\textgreater

Root Cause: \textless cause\textgreater

Fix Description: \textless description\textgreater

Why It Works: \textless explanation\textgreater

\paragraph{Prompts for majority voting} The prompt for majority voting is given as below.

You are an expert in optimization and formal mathematics using the Lean theorem prover. You are given:

- An optimization problem and a corresponding algorithm designed to solve it.

- A candidate Lean formalization of the problem and algorithm.

Please carefully evaluate the candidate formalization and assign a **single numeric score** on a scale from 0 to 100, based on the following criteria:

\#\#\# Scoring Criteria:

1. Files with more errors will receive lower scores. The error message from Lean compilation is given to you below in error messages part. 

2. Score decomposition:

  1) Problem Formalization: Does the Lean code fully and accurately formalize the given optimization problem? (20')
  
  2) Algorithm Correctness: Is the algorithm correctly and rigorously formalized in Lean? (20')
  
  3) Update Scheme Explicitness: Does the formalization explicitly capture the update scheme (iteration rule) used by the algorithm? (20')
  
  4) Theoretical Analysis: Does the formalization include proofs or reasoning about properties of the problem (e.g. convexity, Lipschitz continuity) and the algorithm (e.g. convergence, complexity)? (20')
  
  5) Proof: Is the formalization of the proof complete and nonsorry? (20')
  
3. The score should reflect both syntatic correctness and semantic correctness of the formalization.

4. The use of sorry to omit essential proofs or definitions will negatively impact the score. Additionally, any lack of clarity or ambiguity in the formalization should result in a score reduction.

5. If you want to score 100, please make sure that the formalization is complete, clear, and rigorous, with no missing definitions or proofs. Scoring 100 needs to be justified by the completeness and correctness of the formalization.

Please provide **only the numeric score** in your response. Be **objective, strict, and rigorous** in your evaluation.

Problem and algorithm: \{problem\}

Candidate answer: \{candidate\}

Error messages: \{error\_messages\}
\subsection{Hyperparameters and Project Versions}
All of the numerical experiments are conducted on.
The hyperparameters of main generation is given in Table \ref{tab:hyperparams}. All relevant open-source projects are provided in Table \ref{tab:version}.

\begin{table}[htbp]
\centering
\begin{tabular}{l|l}
\toprule
\textbf{hyperparameter} & \textbf{value / setting} \\
\midrule
model name & deepseek-reasoner \\
max generation steps for backbone & 3 \\
max generation steps for proof & 3 \\
temperature & 0.7 \\
max tokens & 16000 \\
top p & 0.9 \\
frequency penalty & 0.2 \\
retry on failure & yes \\
max correction steps & 3 \\
error knowledge retrieval & 3 entries \\
max attempts per final fix & 2 \\
\bottomrule
\end{tabular}
\caption{Hyperparameters used in SITA}
\label{tab:hyperparams}
\end{table}
\section{Error Feedback Details}
\subsection{Errors Stored in the Knowledge Base}
The error knowledge stores 75 distinct errors with different kinds. In this subsection, we present model-generated fix suggestions for several frequently occurring errors in Lean, demonstrating the effectiveness of the knowledge-based error feedback mechanism in guiding formalization.
\begin{table}[htbp]
\centering
\begin{tabular}{ll}
\toprule
\textbf{Error Type} & \textbf{Count} \\
\midrule
syntax error & 21 \\
type mismatch & 19 \\
failed to synthesize & 7 \\
invalid field & 7 \\
unknown identifier & 7 \\
unexpected token & 4 \\
unknown constant & 2 \\
unclassified & 2 \\
missing definition & 2 \\
timeout & 1 \\
no goals to be solved & 1 \\
tactic 'apply' failed & 1 \\
incomplete proof & 1 \\
\bottomrule
\end{tabular}
\caption{Normalized Error Types and Their Frequencies (Descending Order)}
\label{table: error_kind}
\end{table}

The categorized error types identified in the knowledge base is presented in Table \ref{table: error_kind}. These error types encompass the most frequently encountered issues in the generation of optimization problems and their associated algorithms.

\subsection{Case Study}
\paragraph{Case 1: Failed to synthesize error}
\begin{itemize}
    \item Error message: 
    
    failed to synthesize HSub (Fin m →  $\mathbb{R}$) (EuclideanSpace $\mathbb{R}$ (Fin m)) ?m.7571
    
    Additional diagnostic information may be available using the `set\_option diagnostics true` command.
    \item Original Code:
\begin{lstlisting}
def Balanced_wavelet_problem.f (self : Balanced_wavelet_problem W A b lam κ) (α : EuclideanSpace ℝ (Fin n)) : ℝ :=
  let I_n : Matrix (Fin n) (Fin n) ℝ := 1
  let P : Matrix (Fin n) (Fin n) ℝ := I_n - W * Wᵀ
  let term1 : ℝ := (κ / 2) * ‖P *ᵥ α‖₂ ^ 2
  let term2 : ℝ := (1 / 2) * ‖A *ᵥ (Wᵀ *ᵥ α) - b ‖₂ ^ 2
term1 + term2
\end{lstlisting}
\item Fixed code:
\begin{lstlisting}
def Balanced_wavelet_problem.f (self : Balanced_wavelet_problem W A b lam κ) (α : EuclideanSpace ℝ (Fin n)) : ℝ :=
  let I_n : Matrix (Fin n) (Fin n) ℝ := 1
  let P : Matrix (Fin n) (Fin n) ℝ := I_n - W * Wᵀ
  let term1 : ℝ := (κ / 2) * ‖P *ᵥ α‖₂ ^ 2
  let term2 : ℝ := (1 / 2) * ‖A *ᵥ (Wᵀ *ᵥ α) - (b : EuclideanSpace ℝ (Fin m))‖₂ ^ 2
term1 + term2
\end{lstlisting}
\begin{table*}[htbp]
\centering

\begin{tabular}{l|l|l}
\toprule
\textbf{Name} & \textbf{Github Link} & \textbf{Version}\\
\midrule
Lean4 & \url{https://github.com/leanprover/lean4}& 4.13.0\\
Mathlib 4 & \url{https://github.com/leanprover-community/mathlib4} & \texttt{d7317655e2826dc1f1de9a0c138db2775c4bb841}\\
Optlib & \url{https://github.com/optsuite/optlib}& \texttt{26fcefb31b5aee7f7f5fc7194f462b437282675f}\\

\bottomrule
\end{tabular}
\caption{Versions of open-source projects used in SITA}
\label{tab:version}
\end{table*}

\item Fix suggestion: 

\textbf{Error Type:} Failed to synthesize

\textbf{Root Cause:} The expression 
\(
\mathbf{A} \cdot_v (\mathbf{W}^\top \cdot_v \boldsymbol{\alpha})
\)
produces a vector in \(\mathrm{EuclideanSpace}\,\mathbb{R}^{m}\), but \(\mathbf{b}\) lacks explicit typing. Lean's type inference couldn't automatically determine that \(\mathbf{b}\) must inhabit \(\mathrm{EuclideanSpace}\,\mathbb{R}^{m}\) for the subtraction
\(
\mathbf{A} \cdot_v (\mathbf{W}^\top \cdot_v \boldsymbol{\alpha}) - \mathbf{b}
\)
to be valid.

\textbf{Fix Description:} Added a type ascription
\(
\mathbf{b} : \mathrm{EuclideanSpace}\,\mathbb{R}^{m}
\)
to explicitly specify the type of \(\mathbf{b}\).

\textbf{Why It Works:} The type ascription constrains \(\mathbf{b}\) to \(\mathrm{EuclideanSpace}\,\mathbb{R}^{m}\), matching the output type of
\(
\mathbf{A} \cdot_v (\mathbf{W}^\top \cdot_v \boldsymbol{\alpha}).
\)
This ensures:
\begin{enumerate}
  \item the subtraction operation is well-typed;
  \item the norm \(\lVert \cdot \rVert_2\) is applied to a homogeneous vector space;
  \item type inference succeeds with unambiguous vector space dimensions.
\end{enumerate}
\end{itemize}
\paragraph{Remark} This fix suggestion demonstrates that the model-generated feedback correctly identified a type mismatch due to Lean’s failure to infer the type of b. The suggested ascription (b : EuclideanSpace $\mathbb{R}$ (Fin m)) resolves the issue by aligning b with the expected type of the subtraction. This indicates the model can localize, correct and conclude type errors by leveraging context.

\paragraph{Case 2: Syntax Error}
\begin{itemize}
    \item Error message: 
    
    \texttt{unexpected token 'def'; expected ')', ',' or ':'}.
    
    \item Original Code:
\begin{lstlisting}
def SparseLogisticRegression_problem.f (self : SparseLogisticRegression_problem A b lambda) : EuclideanSpace ℝ (Fin n) → ℝ :=
  fun x ↦ ∑ i, Real.log (1 + Real.exp (-b i * (A i ⬝ᵥ x))
\end{lstlisting}

    \item Fixed code:
\begin{lstlisting}
def SparseLogisticRegression_problem.f (self : SparseLogisticRegression_problem A b lambda) : EuclideanSpace ℝ (Fin n) → ℝ :=
  fun x ↦ ∑ i, Real.log (1 + Real.exp (-b i * (A i ⬝ᵥ x)))
\end{lstlisting}

    \item Fix suggestion: 

\textbf{Error Type:} Syntax error (parenthesis mismatch)

\textbf{Root Cause:} The original expression contains unbalanced parentheses: both \texttt{Real.log} and \texttt{Real.exp} are function applications that open a parenthesis, but only \texttt{Real.exp} is closed. This leaves the \texttt{Real.log} application incomplete, leading to a parser error.

\textbf{Fix Description:} Added a closing parenthesis \texttt{)} at the end of the line to complete the \texttt{Real.log} application.

\textbf{Why It Works:} With the added closing parenthesis, the expression becomes syntactically valid:

\texttt{Real.log(}  \texttt{1 + Real.exp(}\texttt{-b i * (A i $\cdot_v$ x)} \texttt{)} \texttt{)}.
This satisfies Lean’s requirement for balanced delimiters in nested function applications and resolves the parsing failure.
\end{itemize}

\paragraph{Remark} LLMs frequently produce expressions with mismatched parentheses. However, the error messages provided by the Lean compiler are often indirect, it typically reports an “unexpected token” rather than explicitly indicating a missing or mismatched parenthesis. The model  explainer enhances this feedback by accurately identifying the root cause and explicitly pointing out the parenthesis mismatch.

\paragraph{Case 3: Type Mismatch}
\begin{itemize}
    \item Error message:
    
    \texttt{don't know how to synthesize implicit argument '$\delta$' \\
    @huber\_loss (?m.2688 f self x i) (A i $\cdot _v$ x - b i) \\
    context: \\
    m n : $\mathbb{N}$ \\
    A : Matrix (Fin m) (Fin n) $\mathbb{R}$ \\
    b : Fin m → $\mathbb{R}$ \\
    $\delta$ lam : $\mathbb{R}$ \\
    self : RobustRegression\_Huber\_L1\_problem A b $\delta$ lam \\
    x : EuclideanSpace $\mathbb{R}$ (Fin n) \\
    i : Fin m \\
    $\vdash$ $\mathbb{R}$}
    
    \item Original Code:
\begin{lstlisting}
def RobustRegression_Huber_L1_problem.f (self : RobustRegression_Huber_L1_problem A b δ lam) : EuclideanSpace ℝ (Fin n) → ℝ :=
  fun x ↦ ∑ i, huber_loss (A i ⬝ᵥ x - b i)
\end{lstlisting}

    \item Fixed code:
\begin{lstlisting}
def RobustRegression_Huber_L1_problem.f (self : RobustRegression_Huber_L1_problem A b δ lam) : EuclideanSpace ℝ (Fin n) → ℝ :=
  fun x ↦ ∑ i, huber_loss δ (A i ⬝ᵥ x - b i)
\end{lstlisting}

    \item Fix suggestion: 

\textbf{Error Type:} Type mismatch

\textbf{Root Cause:} The \texttt{huber\_loss} function expects two arguments: the threshold \texttt{$\delta$ : $\mathbb{R}$} and a residual \texttt{r : $\mathbb{R}$}. In the original code, only the residual \texttt{(A i $\cdot_v$ x - b i)} is provided, leaving \texttt{$\delta$} implicit. However, Lean cannot infer \texttt{$\delta$} from the context and throws an error.

\textbf{Fix Description:} Explicitly added the missing argument \texttt{$\delta$} before the residual, resulting in \texttt{huber\_loss $\delta$ (A i $\cdot_v$ x - b i)}.

\textbf{Why It Works:} The type signature of \texttt{huber\_loss} is \texttt{$\mathbb{R}$ → $\mathbb{R}$ → $\mathbb{R}$}. By supplying both the threshold and residual explicitly, the function application becomes well-typed, and the compiler can proceed with type-checking the full expression.
\end{itemize}

\paragraph{Remark} This kind of error is common when using functions with implicit arguments in Lean. If the compiler lacks sufficient context to infer such arguments, it reports a cryptic “don't know how to synthesize” error. Providing the expected parameters explicitly can often resolve such issues cleanly.

\section{Optimization Problem Dataset}\label{appendix: dataset}
We provide additional statistical details of our optimization problem dataset. Since each optimization problem can be solved using multiple algorithms, we treat a pair consisting of a problem and an algorithm as a single data point. The dataset includes 42 distinct optimization problems, solved by the following algorithmic families: 9 using proximal gradient methods (PGM), 9 using standard gradient descent (GD), 9 using block coordinate descent (BCD), 8 using Nesterov's acceleration, and 7 using the alternating direction method of multipliers (ADMM). The update schemes and their corresponding algorithm are summarized as follows.
\begin{itemize}
    \item GD: solves unconstrained optimization problem as $\min\limits_xf(x)$. Its update scheme is as 
    \[x_{k+1} = x_k - t \cdot \nabla f(x_k).\]
    \item PGD:  solves composite optimization problem as  
    $\min\limits_xf(x) + g(x)$. Its update scheme is as 
    \[x_{k+1} = \operatorname{prox}_{tg}(x_k - t\nabla f(x_k)).\]
    \item Nesterov: solves composite optimization problem. Its update scheme is as
    \begin{align*}
    \begin{cases}   
      y_k = x_k + \frac{\gamma_k (1 - \gamma_{k - 1})}{\gamma_{k - 1}} (x_k - x_{k - 1}), \\
      x_{k+1} = \operatorname{prox}_{t h}(y_k - t \nabla f(y_k)),
    \end{cases}
  \end{align*}
  where $\gamma_k = \frac{2}{k+2}$.
  \item BCD: solves optimization problems with block structure:
  \begin{align*}
    \underset{x, y}{\min} \; \Psi(x,y) = f(x) + g(y) + H(x,y).
\end{align*}
The corresponding algorithm update gives as:
\begin{align*}
  x^{k+1} &\in \text{prox}_{c_k f} \left( x^k - c_k \nabla_x H(x^k, y^k) \right), \\
    y^{k+1} &\in \text{prox}_{d_k g} \left( y^k - d_k \nabla_y H(x^{k+1}, y^k) \right).
\end{align*}
\item ADMM: solves optimization problems with linear constraint as below:
\begin{equation}
    \begin{aligned}
        \min_{x_1, x_2} & \quad f_1(x_1) + f_2(x_2), \\
        \text{s.t.} & \quad A_1 x_1 + A_2 x_2 = b.
    \end{aligned}
    \label{eq: ADMM}
\end{equation}
The update scheme is given as:
\begin{align*}
        x_{1}^{k+1} &= \underset{x_1}{\arg \min}\; L_{\rho}(x_1,x_2^k,y^k),\\
        x_{2}^{k+1} &= \underset{x_2}{\arg \min}\; L_{\rho}(x_1^{k+1},x_2,y^k),\\
        y^{k+1} &= y^k + \tau \rho (A_1x_1^{k+1} + A_2x_2^{k+1} - b), 
\end{align*}
where $L_{\rho}(x_1, x_2, y)$ denotes the augmented Lagrangian function and is given as:
\begin{align*}
  L_{\rho}(x_1, x_2, y) &= f_1(x_1) + f_2(x_2) +\left<y, A_1x_1 + A_2x_2 - b\right>
   \\ & + \frac {\rho} 2  \left\lVert A_1x_1 + A_2x_2 - b\right \rVert^2.  
\end{align*}
\end{itemize}
The complexity of these algorithms also differs, with ADMM being the most complicated owing to its reliance on the augmented Lagrangian formulation, while PGM and GD represent more elementary optimization schemes. The distribution is given in Figure \ref{fig:problem dis}.
\begin{figure}[htbp]
    \centering
    \includegraphics[width=0.7\linewidth]{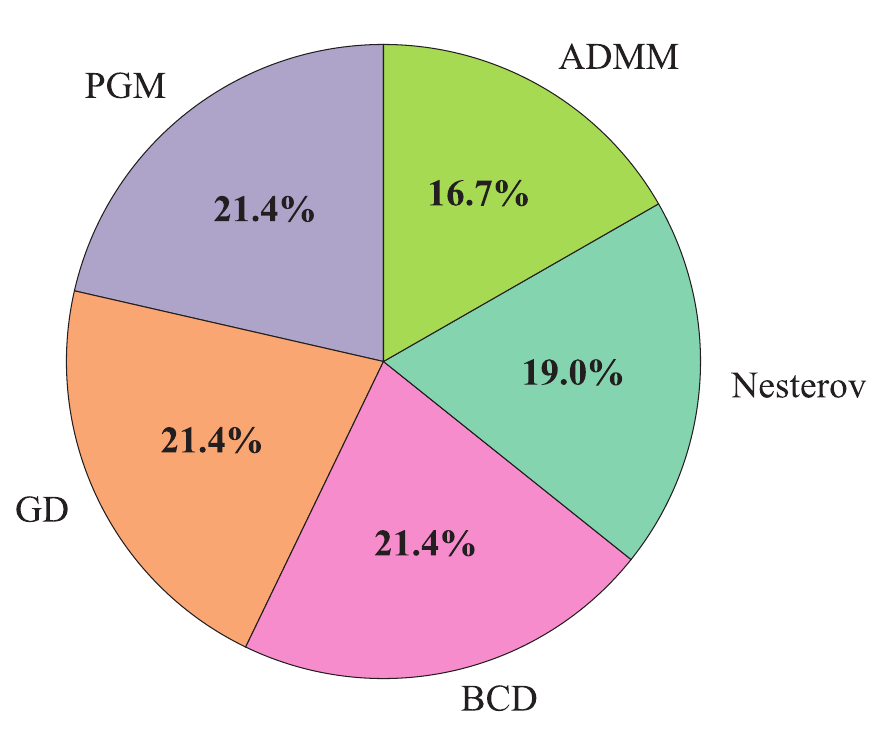}
    \caption{Distribution of algorithm types in the dataset}
    \label{fig:problem dis}
\end{figure}

While each algorithmic family may contain several variants, we select a representative implementation from each to ensure consistency and simplicity. Variants within the same family tend to behave similarly in structure and behavior, and thus are not distinguished separately in this dataset. 

The problems in our dataset primarily focus on optimization over vector-valued variables. Some of these problems are relatively challenging to specify, and computing their gradients requires careful derivation. As a result, formalizing their definitions and structures is non-trivial. Within each algorithmic class, the difficulty of the problems is relatively uniform. For clarity and better intuition, we present additional representative problems below.
\begin{itemize}
    \item Group Lasso %\cite{Meier2008Group}
    (regularized least-squares problem with group sparsity)
  \begin{align}\label{eq:group lasso}
       \min_{x \in \mathbb{R}^n} \frac{1}{2}\|Ax - b\|_2^2 + \lambda \sum_{g=1}^G \|x_{G_g}\|_2,
  \end{align}
  where \( A \in \mathbb{R}^{m \times n} \), \( b \in \mathbb{R}^m \), \( \lambda > 0 \), and \( \{G_1, \dots, G_G\} \) is a partition of the index set \( \{1, \dots, n\} \). Each term \( \|x_{G_g}\|_2 \) promotes group-level sparsity in the solution.
  \item Robust regression:% \cite{Huber1992}: 
  This formulation addresses robustness to outliers in the data by employing the Huber loss function and promotes sparsity in the regression coefficients through an $\ell_1$ regularization term.
  \[
    \min_{x \in \mathbb{R}^n} \sum_{i=1}^m \phi_\delta(a_i^\top x - b_i) + \lambda \|x\|_1,
  \]
  where \( A = [a_1^\top; \dots; a_m^\top] \in \mathbb{R}^{m \times n} \), \( b \in \mathbb{R}^m \), \( \lambda > 0 \), and the Huber loss \( \phi_\delta(r) \) is defined as
  \[
    \phi_\delta(r) =
    \begin{cases}
      \frac{1}{2} r^2, & \text{if } |r| \leq \delta, \\
      \delta (|r| - \frac{1}{2} \delta), & \text{otherwise}.
    \end{cases}
  \]
  \item L2 Group Sparse Coding (sparse and structured representation model)
  \[
    \min_{x \in \mathbb{R}^n,\, z \in \mathbb{R}^m} \frac{1}{2} \|Ax + Bz - y\|_2^2 + \lambda_1 \|x\|_1 + \lambda_2 \|z\|_2,
  \]
  where \( y \in \mathbb{R}^d \) is the observed signal, \( A \in \mathbb{R}^{d \times n} \) and \( B \in \mathbb{R}^{d \times m} \) are dictionary matrices, and \( x \in \mathbb{R}^n \), \( z \in \mathbb{R}^m \) are latent representations. The term \( \|x\|_1 \) enforces sparsity in \( x \), while the term \( \|z\|_2 \) encourages structured or group-like behavior in \( z \). This formulation is widely used in sparse coding, multi-component signal decomposition, and representation learning tasks.

Some related conditions needed to prove and formalize are listed as follows.
\begin{itemize}
    \item Explicit calculation of the gradient or subgradient of the target functions.
    \item The properties of the target function, including boundedness, convexity, differentiability, and the Lipschitz continuous of the gradient.
    \item The explicit solution of the subproblem, including the solution of the proximal operator of the target function.
    \item The Kurdyka–Łojasiewicz (KL) property of the target functions.
\end{itemize}

\end{itemize}
\section{Case Study of SITA Output}
\subsection{Case Study of Definition Generation}
We present several classical generation examples to demonstrate the effectiveness of the SITA framework.
\paragraph{Case 1: (Success) Logistic regression with $\ell_1$ penalty}
The logistic regression with $\ell_1$ penalty solves the problem as follows. 
\begin{align*}
    \min_x \psi(x) =  \sum_{i=1}^m \log\bigl(1+e^{-b_i a_i^\top x}\bigr) + \lambda \|x\|_1.
\end{align*}
The formalization of the logistic regression problem involves formalizing the definition of the target function $f(x) =\sum_{i=1}^m  \log\bigl(1+e^{-b_i a_i^\top x}\bigr)$ and calculate the gradient of the target function as \[
\nabla f(x) = \sum_{i=1}^m \frac{-b_i a_i}{1 + e^{-b_i a_i^\top x}}.
\] 
Instead of following human notations using $a_i$ as vectors, the autoformalization mechanism applies transforming the form of the target function considering all $a_i$ as an entire matrix $A = [a_1^\top, \cdots, a_m^\top ]^\top$, simplifying the notations. The model correctly formalizes the definitions of this problem in the following type class. 
\begin{lstlisting}
class SparseLogisticRegression_problem {m n : ℕ} (A : Matrix (Fin m) (Fin n) ℝ) (b : (Fin m) → ℝ) (lambda : ℝ) where
  hA : A ≠ 0
  hlambda : lambda > 0

def SparseLogisticRegression_problem.f (self : SparseLogisticRegression_problem A b lambda) : EuclideanSpace ℝ (Fin n) → ℝ :=
  fun x ↦ ∑ i, Real.log (1 + Real.exp (-b i * (A *ᵥ x) i))

def SparseLogisticRegression_problem.g (self : SparseLogisticRegression_problem A b lambda) : EuclideanSpace ℝ (Fin n) → ℝ :=
  fun x ↦ lambda * ‖ x ‖₁

def SparseLogisticRegression_problem.target (self : SparseLogisticRegression_problem A b lambda) : EuclideanSpace ℝ (Fin n) → ℝ :=
  fun x ↦ self.f x + self.g x
\end{lstlisting}

Apart from definitions, the SITA framework also enhance the model's ability on calculation. Constructing intermediate variables as $u$, $d$, $grad$, the model correctly gives the formalized explicit update form of the proximal gradient descent for logistic regression with $\ell_1$ penalty in the following algorithm class. The algorithm is correct not only syntactically in Lean, but also mathematically.
\begin{lstlisting}
class proximal_gradient_method_SparseLogistic (pro : SparseLogisticRegression_problem A b lambda) (x₀ : EuclideanSpace ℝ (Fin n)) where
  t : ℝ
  x : ℕ → EuclideanSpace ℝ (Fin n)
  y : ℕ → EuclideanSpace ℝ (Fin n)
  ht : t > 0
  update1 : ∀ k : ℕ,
    let u : Fin m → ℝ := A *ᵥ x k
    let d : Fin m → ℝ := fun i ↦ -b i * (1 / (1 + Real.exp (b i * u i)))
    let grad : EuclideanSpace ℝ (Fin n) := Aᵀ *ᵥ d
    y k = x k - t • grad
  update2 : ∀ (k : ℕ), ∀ i, x (k + 1) i = Real.sign (y k i) * max (|y k i| - t * lambda) 0
  initial : x 0 = x₀
\end{lstlisting}
This case shows that the model has the basic ability to formulate the concrete expressions of the optimization problem, such as the target functions or gradient. It can also fit in the target template of the abstract structure.
\paragraph{Case 2: (Success) $\ell_2$ regularized least square problem} The $\ell_2$ regularized least square problem is defined as 
\begin{align*}
    \min_x \frac{1}{2}\|Ax-b\|^2 + \|x\|_2.
\end{align*}
\begin{lstlisting}
class proximal_gradient_method_L2RegLS (pro : L2_regularized_least_squares_problem A b mu) (x₀ : EuclideanSpace ℝ (Fin n)) where
  t : ℝ
  x : ℕ → EuclideanSpace ℝ (Fin n)
  ht : t > 0
  update : ∀ k : ℕ,
    let grad : EuclideanSpace ℝ (Fin n) := Aᵀ *ᵥ (A *ᵥ (x k) - b)
    let y : EuclideanSpace ℝ (Fin n) := x k - t • grad
    x (k + 1) =
      if ‖ y ‖₂ ≤ t * mu then 0
      else (1 - (t * mu) / ‖ y ‖₂) • y
  initial : x 0 = x₀
\end{lstlisting}
The model successfully formalizes the explicit form of the proximal operator of the $\ell_2$ norm:
\[
\operatorname{prox}_{t\mu \|\cdot\|_2}(y) =
\begin{cases}
0, & \text{if } \|y\|_2 \leq t\mu, \\
\left(1 - \frac{t\mu}{\|y\|_2} \right) y, & \text{otherwise},
\end{cases}
\]
and integrates it into the iterative process of the algorithm. In each iteration, the model computes the gradient of the smooth part of the objective function, $\frac{1}{2}\|Ax - b\|^2$, and performs a gradient step followed by applying the proximal operator to $y$ to handle the nonsmooth term $\mu \|x\|_2$.

Throughout the process, the model incorporates mathematical reasoning and successfully formalizes the corresponding mathematical expressions. Syntactically, it employs a conditional structure using \texttt{if ... then ... else ...}, which is a relatively advanced but standard construct in Lean to handle piecewise-defined functions. This ensures the correctness and executability of the formal definition of the $\ell_2$ proximal operator.

\paragraph{Case 3: (Success) Group Lasso with Nesterov Acceleration} The formalization of definitions in the definition of group Lasso problem is harder. The form of the target function is as \eqref{eq:group lasso}. Here, we provide two versions of the definition of model. Both of the generations are correct. However, they utilize different ways to define the groups. The first one of them is slightly more general. 
\begin{lstlisting}
class Group_Lasso_problem {m n G : ℕ} (A : Matrix (Fin m) (Fin n) ℝ) (b : (Fin m) → ℝ) (lam : ℝ) (groupOf : Fin n → Option (Fin G)) where
  hA : A ≠ 0
  hlam : lam > 0

def Group_Lasso_problem.g (pro : Group_Lasso_problem A b lam groupOf) : EuclideanSpace ℝ (Fin n) → ℝ :=
  fun x ↦ lam * ∑ g : Fin G, ‖ fun j ↦ if groupOf j = some g then x j else 0 ‖₂
\end{lstlisting}
This formalization captures the regularization term of the group Lasso objective: \( \lambda \sum_{g=1}^{G} \|x_{g}\|_2, \)
where each group \( x_g \) corresponds to the components of \( x \in \mathbb{R}^n \) indexed by \texttt{groupOf}.The function \texttt{groupOf : Fin n → Option (Fin G)} assigns each coordinate to a group (\texttt{some g}) or excludes it (\texttt{none}). The term \texttt{Option} deals with indexes which are not in any group. Features not assigned to a group are excluded from the penalty. The condition \texttt{groupOf j = some g} identifies coordinates belonging to group $g$, enabling masked projection for computing group-wise $\ell_2$ norms.

\begin{lstlisting}
class Nesterov_Group_Lasso (pro : Group_Lasso_problem A b lam groupOf) (x₀ : EuclideanSpace ℝ (Fin n)) where
  hl : pro.l > (0 : ℝ)
  x : ℕ → EuclideanSpace ℝ (Fin n)
  y : ℕ → EuclideanSpace ℝ (Fin n)
  w : ℕ → EuclideanSpace ℝ (Fin n)
  t : ℕ → ℝ
  γ : ℕ → ℝ
  oriy : y 0 = x 0
  initial : x 0 = x₀
  teq : ∀ n : ℕ, t n = 1 / pro.l
  γeq : ∀ n : ℕ, γ n = 2 / (2 + n)
  update1 : ∀ k : ℕ+, y k = x k + (γ k * (1 - γ (k - 1)) / (γ (k - 1))) • (x k - x (k - 1))
  update2 : ∀ k : ℕ,
      let grad : EuclideanSpace ℝ (Fin n) := Aᵀ *ᵥ (A *ᵥ y k - b)
      w k = y k - t k • grad
  update3 : ∀ k : ℕ, x (k + 1) =
      fun j =>
        match groupOf j with
        | none => w k j
        | some g =>
            let s := (∑ j', if groupOf j' = some g then (w k j')^2 else 0) ^ (1/2)
            let factor := if s = 0 then 0 else max (1 - t k * lam / s) 0
            factor * w k j
\end{lstlisting}
The \texttt{update3} clause encodes the group-wise proximal operator for the group $\ell_2$ penalty. For each coordinate \( j \), the update distinguishes two cases:
\begin{itemize}
  \item If \texttt{groupOf j = none}, coordinate \( j \) is not penalized and is directly set as \( x_{k+1}(j) := w_k(j) \).
  \item If \texttt{groupOf j = some g}, coordinate \( j \) belongs to group \( g \), and an operator is applied:
  \[
  x_{k+1}(j) := \left( \max\left(1 - \frac{t_k \lambda}{\|w_k^{(g)}\|_2},\, 0 \right) \right) \cdot w_k(j),
  \]
  where \( w_k^{(g)} \) denotes the subvector of \( w_k \) restricted to group \( g \).
\end{itemize}
This is the exact proximal mapping of the group Lasso penalty \( \lambda \sum_g \|x_g\|_2 \) evaluated at \( w_k \). This is formalized with the ``match" pattern in formalization.

Another trial generates the following definitions of group lasso problem.
\begin{lstlisting}
class Group_Lasso_problem (A : Matrix (Fin m) (Fin n) ℝ) (b : Fin m → ℝ) (lam : ℝ) (G : ℕ) (group_index : Fin n → Fin G) where
  hA : A ≠ 0
  hlam : lam > 0

def Group_Lasso_problem.g (pro : Group_Lasso_problem A b lam G group_index) : EuclideanSpace ℝ (Fin n) → ℝ :=
  fun x ↦ lam * ∑ g : Fin G, ‖ fun i : Fin n => if group_index i = g then x i else 0 ‖₂
\end{lstlisting}

\begin{lstlisting}
class Nesterov_Group_Lasso (pro : Group_Lasso_problem A b lam G group_index) (x₀ : EuclideanSpace ℝ (Fin n)) where
  hl : pro.l > (0 : ℝ)
  x : ℕ → EuclideanSpace ℝ (Fin n)
  y : ℕ → EuclideanSpace ℝ (Fin n)
  w : ℕ → EuclideanSpace ℝ (Fin n)
  t : ℕ → ℝ
  γ : ℕ → ℝ
  oriy : y 0 = x 0
  initial : x 0 = x₀
  teq : ∀ n : ℕ, t n = 1 / pro.l
  γeq : ∀ n : ℕ, γ n = 2 / (2 + n)
  update1 : ∀ (k : ℕ+), y k = x k + (γ k * (1 - γ (k - 1)) / (γ (k - 1))) • (x k - x (k - 1))
  update2 : ∀ k : ℕ,
    let grad : EuclideanSpace ℝ (Fin n) := Aᵀ *ᵥ (A *ᵥ y k - b)
    w k = y k - t k • grad
  update3 : ∀ k : ℕ, x (k + 1) =
    fun i =>
      let g_i := group_index i
      let nrm := ‖ fun j : Fin n => if group_index j = g_i then w k j else 0 ‖₂
      if nrm = 0 then 0 else
        let scale := max (1 - t k * lam / nrm) 0
        scale * w k i
\end{lstlisting}
The second version employs simpler tactics and performs a case-by-case analysis. Upon manual verification, its results align with the correct target functions and the Nesterov acceleration method applied to the group Lasso problem.

\paragraph{Case 4: (Failed) Binary classification with exponential loss} There is an unfixable error occurred in the formalized definitions of binary classification with exponential loss. %\cite{zhao2010convex}. 
The optimization problem gives as follows:
\begin{align*}
    \min_x f(x) = \sum_{i=1}^n \exp (-b_i a_i^\top x) + \frac{\lambda}{2} \|x\|^2.
\end{align*}
The model also transform the vectors $a_i$ into a matrix and gives the following formalized definitions.

\begin{lstlisting}
def BinaryExpClass_problem.f {m n : ℕ} (A : Matrix (Fin m) (Fin n) ℝ) (b : Fin m → ℝ) (lam : ℝ) :
    EuclideanSpace ℝ (Fin n) → ℝ :=
  fun x ↦ (∑ i, exp (- b i * (A *ᵥ x) i)) + (lam / 2) * ‖ x ‖₂ ^ 2

class BinaryExpClass_GD {m n : ℕ} {A : Matrix (Fin m) (Fin n) ℝ} {b : Fin m → ℝ} {lam : ℝ}
  (pro : BinaryExpClass_problem A b lam) (x₀ : EuclideanSpace ℝ (Fin n)) where
  t : ℝ
  x : ℕ → EuclideanSpace ℝ (Fin n)
  ht : t > 0
  update : ∀ k : ℕ,
    let gra : EuclideanSpace ℝ (Fin n) := Aᵀ *ᵥ (fun i ↦ - b i * exp (- b i * (A *ᵥ x k) i)) + lam • x k
    x (k + 1) = x k - t • gra
  initial : x 0 = x₀
\end{lstlisting}

An error occurs at line 11 of the code due to a type mismatch. Specifically, we cannot directly add the variable \texttt{x k}, which has type \textbf{EuclideanSpace} $\mathbb{R}$ \texttt{(Fin n)}, to an expression of type \texttt{Fin n} $\to$ $\mathbb{R}$. Besides, the matrix vector multiplication is perform on the type \texttt{Fin n} $\to$ $\mathbb{R}$.  Although relevant information is present in the error knowledge, the model fails to resolve the issue, partly due to Lean’s complex and strict type coercion system. A common human workaround is to explicitly convert the form expression into the type \textbf{EuclideanSpace} $\mathbb{R}$ \texttt{(Fin n)} using a \texttt{let} binding for type transformation.

\paragraph{Case 5: (Failed) TV denoising with constraint splitting}
The total variation (TV) denoising gives as:
\begin{align*}
    \min_{x,z} \|x-b\|^2 + \|z\|_1, \quad \text{s.t.} \quad z=Dx.
\end{align*}
The problem can be solved via ADMM. However, the model fails to specialize the map $A_1$ in \eqref{eq: ADMM}. The statement of the general structure of ADMM is also defined on general Hilbert spaces. Hence the matrix here is generalized as continuous linear map on Hilbert spaces. The formalization code is given as below.
\begin{lstlisting}
class TV_denoising (n m : ℕ) (lam : ℝ) (D : Matrix (Fin m) (Fin n) ℝ) (b : EuclideanSpace ℝ (Fin n)) where
  hlam : lam > 0
  hD : D ≠ 0

def TV_denoising.A₁ (self : TV_denoising n m lam D b) : EuclideanSpace ℝ (Fin n) →L[ℝ] EuclideanSpace ℝ (Fin m) :=
  (Matrix.mulVecLin D).toContinuousLinearMap
\end{lstlisting}
For structure-to-instance task, the model needs to transform the type of the matrix $D$ as \texttt{Matrix (Fin n) (Fin n) $\mathbb{R}$} to a continuous linear map with type \texttt{EuclideanSpace $\mathbb{R}$ (Fin n) $\to$L[$\mathbb{R}$] EuclideanSpace $\mathbb{R}$ (Fin m)}. Type transformation is also challenging for human experts. The model attempts to leverage some functions from Mathlib4 to perform the transformation but fails. However, it comes close to succeeding, as the correct answer is as follows.
\begin{lstlisting}
def TV_denoising.A₁ (self : TV_denoising n m lam D b) : EuclideanSpace ℝ (Fin n) →L[ℝ] EuclideanSpace ℝ (Fin m) :=
  LinearMap.toContinuousLinearMap (Matrix.mulVecLin D)
\end{lstlisting}
It is also hard for the model to understand the type error as "invalid field notation, function 'LinearMap.toContinuousLinearMap' does not have argument with type (LinearMap ...) that can be used, it must be explicit or implicit with a unique name". 

From this perspective, enabling the model to flexibly perform type transformations requires not only a solid familiarity with the Mathlib library, but also a thorough understanding of possible error messages and usage patterns.

\paragraph{Conclusion}
This case study illustrates the strong capabilities of the SITA framework in generating formal definitions and algorithmic procedures for a variety of optimization problems in the Lean proof assistant. One of the most notable discovering is that the model can autonomously define new auxiliary functions and integrate them into the abstract structure. This demonstrates a promising degree of creativity and abstraction in its formal reasoning process. For example, it not only defines the objective function and regularization terms correctly but also follows through by formalizing the corresponding optimization algorithms within the same structural framework.

The model shows meta-programming capability in Lean. In the group Lasso example, it makes use of advanced tactics such as pattern matching, group-wise norm projection, and conditionals to define piecewise and masked operations. This reflects a deepening understanding of Lean's expressive power and modular design patterns. However, a critical challenge remains: the model often struggles with Lean’s strict and nuanced type system. Particularly in the failed cases, the model generates expressions with type mismatches, a common difficulty even for human users of Lean. The inability to perform the correct type coercion, despite having access to the relevant knowledge, highlights a gap in type inference and prediction under Lean’s formal logic constraints.

\subsection{Case Study of Proof Generation}

\paragraph{Case 1: (Success) the proof of properties in residual $\ell_2$ decomposition}
The problem gives as 
\begin{align*}
    \min_{x,r} \psi(x,r) =  \|x+r-y\|_2^2 + \lambda_1 \|x\|_1 + \lambda_2\|r\|_2
\end{align*}
To prove the convergence, one needs to establish that $\psi(x,r)$ is lower bounded. The model generates two slightly different but both correct proofs. We give the both codes here. The first is as below. 
\begin{lstlisting}
class ResidualL2Decomposition_problem (y : EuclideanSpace ℝ (Fin n)) (lambda1 lambda2 : ℝ) where
  hlambda1 : lambda1 > 0
  hlambda2 : lambda2 > 0

def ResidualL2Decomposition_problem.g (pro : ResidualL2Decomposition_problem y lambda1 lambda2) :
    EuclideanSpace ℝ (Fin n) → ℝ := fun r ↦ lambda2 * ‖ r ‖₂

lemma ResidualL2Decomposition_problem.lbdg (pro : ResidualL2Decomposition_problem y lambda1 lambda2) :
    BddBelow (pro.g '' Set.univ) := by
  use 0
  rintro _ ⟨z, _, rfl⟩
  unfold BCD_Residual_L2_Decomposition.ψ
  apply add_nonneg
  · apply add_nonneg
    · rw [Residual_L2_Decomposition_problem.f]
      apply mul_nonneg
      exact le_of_lt pro.hlambda1
      apply norm_nonneg
    · rw [Residual_L2_Decomposition_problem.g]
      apply mul_nonneg
      exact le_of_lt pro.hlambda2
      apply norm_nonneg
  · rw [Residual_L2_Decomposition_problem.H]
    apply mul_nonneg
    norm_num
    apply pow_two_nonneg
\end{lstlisting}

This method mainly utilizes the theorem \texttt{add\_nonneg}. However, the next formalization proves from another aspect.
\begin{lstlisting}
lemma BCD_Residual_L2_Decomposition.lbdψ (alg : BCD_Residual_L2_Decomposition pro x0 r0) : BddBelow (alg.ψ '' univ) := by
  obtain ⟨a, ha⟩ := pro.lbdf
  obtain ⟨b, hb⟩ := pro.lbdg
  use a + b
  rintro c ⟨z, _, hz⟩
  rw [← hz]
  dsimp [BCD_Residual_L2_Decomposition.ψ]
  have h1 := ha (mem_image_of_mem pro.f (Set.mem_univ z.1))
  have h2 := hb (mem_image_of_mem pro.g (Set.mem_univ z.2))
  have : 0 ≤ pro.H z := by
    unfold Residual_L2_Decomposition_problem.H
    positivity
  linarith
\end{lstlisting}
\begin{table*}[htbp]
\centering
\begin{tabular}{llccccc}
\toprule
\textbf{Model} & \textbf{Algorithm} & \textbf{Definition} & \textbf{Theorem} & \textbf{Instance} & \textbf{File} & \textbf{MV} \\
\midrule
\multirow{5}{*}{Direct-V3} 
 & GD       & 50.00\%    & 39.80\%   & 40.24\%   & 0.00\%     & 49.47   \\
 & PGM      & 22.73\%    & 8.70\%    & 11.27\%   & 0.00\%     & 47.53   \\
 & Nesterov & 16.39\%    & 10.91\%   & 2.53\%    & 0.00\%     & 48.67   \\
 & BCD      & 29.03\%    & 0.00\%    & 0.00\%    & 0.00\%     & 54.25   \\
 & ADMM     & 15.58\%    & 2.86\%    & 0.00\%    & 0.00\%     & 50.97  \\  
\midrule
\multirow{5}{*}{Direct-R1} 
 & GD       & 21.10\%    & 5.66\%    & 14.29\%    & 0.00\%    & 47.10   \\
 & PGM      & 34.31\%    & 1.35\%    & 15.25\%    & 0.00\%    & 38.51   \\
 & Nesterov & 38.76\%    & 26.87\%   & 24.53\%    & 0.00\%    & 44.60   \\
 & BCD      & 48.23\%    & 24.86\%   & 3.92\%     & 0.00\%    & 51.01   \\
 & ADMM     & 47.55\%    & 4.00\%    & 6.25\%     & 0.00\%    & 49.07    \\
\midrule
\multirow{5}{*}{SITA-V3} 
 & GD       & 83.33\%          & 85.19\%    & \textbf{100}\%    & 33.33\% & 69.96   \\
 & PGM      & 89.19\%          & 82.72\%    & 88.89\%           & 33.33\% & 70.90   \\
 & Nesterov & 93.75\%          & 83.87\%    & 87.50\%           & 25.00\% & 64.34   \\
 & BCD      & \textbf{94.55}\% & 99.10\%    & \textbf{100.00}\% & 44.44\% & 73.23   \\
 & ADMM     & \textbf{94.29}\%   & 82.41\%    & 77.78\%           & 0.00\%  & 48.98   \\
\midrule
\multirow{5}{*}{SITA-R1} 
 & GD       & \textbf{86.36}\%  & \textbf{89.87}\%  & 94.44\%           & \textbf{44.44}\% & \textbf{79.65} \\
 & PGM      & \textbf{100.00}\% & \textbf{98.73}\%  & \textbf{100.00}\% & \textbf{77.78}\% & \textbf{81.35} \\
 & Nesterov & \textbf{96.94}\%  & \textbf{98.41}\%  & \textbf{96.75}\%  & \textbf{62.50}\% & \textbf{75.56} \\
 & BCD      & 94.37\%           & \textbf{100.00}\% & \textbf{100.00}\% & \textbf{66.67}\% & \textbf{80.02} \\
 & ADMM     &  91.43\%          & \textbf{90.76}\%  &  \textbf{85.71}\% & \textbf{28.57}\% & \textbf{65.23} \\
\bottomrule
\end{tabular}
\caption{
Formalization completion rate comparison. 
\textbf{Direct-V3}: direct generation with DeepSeek-V3. 
\textbf{Direct-R1}: direct generation with DeepSeek-R1. 
\textbf{SITA-V3}: our framework using DeepSeek-V3.  
\textbf{SITA-R1}: our framework using DeepSeek-R1. 
\textbf{Definition}, \textbf{Theorem}, \textbf{Instance} and \textbf{File} denote the criteria for the evaluation of entire file generation. 
\textbf{MV}: results from majority voting.
}
\label{table: big completion}
\end{table*}
\paragraph{Remark} The model proves that the function $g(r) = \lambda_2 \|r\|_2$ is lower bounded with the assumption that $\lambda_2 > 0$. Although wrapped in the definitions of \texttt{ResidualL2Decomposition\_problem.g}, the model knows the target using the tactic ``unfold". The model can generate tactics using ``use", ``rintro", ``unfold", ``apply", ``exact". It also utilizes theorems including \texttt{mul\_nonneg}, \texttt{norm\_nonneg} correctly. This case shows that the model has the basic ability to prove assumptions.

\paragraph{Case 2: (Success )The main theorem for joint sparse coding}
The joint sparse coding solves the following optimization problem as 
\begin{align*}
    \min_{x,y} \|Ax+By-b\|^2 +\lambda_1\|x\|_1 + \lambda_2\|y\|_1.
\end{align*}
The problem can be solved using BCD method. The model can accommodate the local variable definition, $x,y \in \mathbb{R}^n$ into the abstract setting, where the variables are only assumed in general Hilbert spaces.

\begin{table*}[htbp]
\centering
\begin{tabular}{llcccccc}
\toprule
\textbf{Pass@Num} & \textbf{Criteria} &\textbf{PGM} & \textbf{GD} & \textbf{BCD} & \textbf{Nesterov} & \textbf{ADMM} & \textbf{Overall}\\
\midrule
\multirow{3}{*}{Pass@1} 
 & FS       &  55.56\%    & 33.33\%   & 44.44\%   & 50.00\% & 0.00\%   & 38.09\%  \\
 & SC      & 56.76\%    & 72.43\%    & 75.29\%   & 85.00\%  & 28.09\%     & 65.96\% \\
 & PS   & 26.86\%   & 35.72\%    &  24.29\%  & 52.50\%   & 0.00\%  & 27.37\% \\
\midrule
\multirow{3}{*}{Pass@2} 
 & FS       & 66.67\%    & 33.33\%   & 55.56\%   & 62.50\% & 14.28\%  & 47.62\%   \\
 & SC      & 76.42\%    & 81.82\%    & 80.21\%   & 91.66\% & 38.09\% & 73.10\%      \\
 & PS  & 40.28\%    & 47.62\%   & 47.23\%   & 52.5\% & 15.00\%  &41.45\%    \\
\midrule
\multirow{3}{*}{Pass@3} 
& FS       & 77.78\%    & 44.44\%   & 66.67\%   & 62.50\%  & 28.57\%& 57.14\%   \\
 & SC      & 98.10\%    & 83.36\%    & 88.82\%   & 97.80\% & 85.02\% & 90.72\%        \\
 & PS & 62.96\%    & 53.77\%   & 50.55\%    & 63.28\% & 20.00\% & 51.23\%      \\
\bottomrule
\end{tabular}
\caption{
Performance on each algorithm class versus the pass number.
\textbf{FS}: File success compilation rate;
\textbf{SC}: syntactic correctness rate of definitions and statements in the failed cases; 
\textbf{PS}: proof completion rate (\%) in the success cases;
}
\label{table: pass}
\end{table*}
\begin{lstlisting}
theorem Joint_sparse_coding_Convergence_to_critpt (γ : ℝ) (hγ : γ > 1)
    (ck : ∀ k, alg.c k = 1 / (γ * pro.l)) (dk : ∀ k, alg.d k = 1 / (γ * pro.l)) :
    ∃ z_ : (WithLp 2 (EuclideanSpace ℝ (Fin n) × EuclideanSpace ℝ (Fin m))),
      z_ ∈ (critial_point alg.ψ) ∧ Tendsto alg.z atTop (nhds z_) := by
  apply Convergence_to_critpt (alg := alg.BCD) γ hγ ck dk alg.bd alg.hψ alg.lbdψ 
\end{lstlisting}
The model can generate the final convergence theorem as above. There are bunch of lemmas generated by the system relating to the problem properties. The system can also provide a structure to instance kind proof, which just leverage the theorems in the template file. 

\paragraph{Case 3: (Failed) Properties in ridge regression with variable splitting}
This optimization problem is as
\begin{align*}
    \min_{x,z} \frac{1}{2}\|x-c\|^2 + \frac{\lambda}{2}\|z\|_2^2, \quad \text{s.t.} \quad x=z.
\end{align*}
The problem can be addressed using ADMM. The corresponding augmented Lagrangian function gives as
\begin{align*}
    L_\rho(x,z,y) = \frac{1}{2}\|x-c\|^2 + \frac{\lambda}{2}\|z\|_2^2 + \langle y, x-z\rangle + \frac{\rho}{2} \|x-z\|^2.
\end{align*}
One of the key properties we aim to verify is the uniqueness of the solution to the subproblem $\min_x L_(x, \Tilde{z}, \Tilde{y})$ for any $\Tilde{z}, \Tilde{y} \in \mathbb{R}^n$. This can be proved via direct calculation or applying the strongly convexity of the augmented Lagrangian function. Both of the ways need to do long chain reasoning.

The model struggles to complete this proof, encountering difficulties in formalizing the necessary arguments. Nevertheless, it eventually fails to generate a complete proof. The following excerpt presents the proof trace before the application of the harmless correction stage. Due to errors in this early version, the post-harmless version is significantly reduced in length.
\begin{lstlisting}
lemma RidgeSplit.uniq_subX (self : RidgeSplit lam c) (z : EuclideanSpace ℝ (Fin n)) (y : EuclideanSpace ℝ (Fin n)) :
    Problem.UniqSubX self.isProblem ρ z y := by
  let d := z - y
  let x0 := (1/(1+ρ)) • (c + ρ • d)
  have h_expand : ∀ x, self.isProblem.augLag ρ x z y = self.isProblem.augLag ρ x0 z y + (1+ρ)/2 * ‖ x - x0 ‖₂^2 := by
    intro x
    let δ := x - x0
    have hc : x0 - c = ρ • (d - x0) := by
      field_simp [x0, d]
      rw [smul_smul, mul_comm, mul_inv_cancel (by linarith), one_smul]
      linear_combination (1+ρ) • x0 - c - ρ • d
    unfold Problem.augLag
    simp only [RidgeSplit.A₁, RidgeSplit.A₂, RidgeSplit.bv, ContinuousLinearMap.coe_id', id_eq, Pi.add_apply,
      ContinuousLinearMap.neg_apply, Pi.neg_apply, add_zero, sub_zero, add_neg_eq_zero]
    unfold RidgeSplit.f₁
    simp only
    rw [← sub_add_sub_cancel x c x0, ← sub_add_sub_cancel x d x0]
    rw [norm_add_sq, norm_add_sq]
    rw [norm_sub_sq, norm_sub_sq]
    simp only [inner_add_left, inner_add_right, inner_smul_right, inner_sub_left, inner_sub_right]
    rw [hc]
    ring_nf
    simp only [smul_sub, sub_add_sub_cancel, sub_self, inner_zero_left, inner_zero_right, zero_add]
    ring
  have h_x0_min : IsMinOn (self.isProblem.augLag ρ · z y) univ x0 := by
    intro x _
    rw [h_expand x]
    nlinarith
  have h_unique : ∀ x', IsMinOn (self.isProblem.augLag ρ · z y) univ x' → x' = x0 := by
    intro x' h_min
    have h_eq : self.isProblem.augLag ρ x' z y = self.isProblem.augLag ρ x0 z y :=
      le_antisymm (h_min x0 (mem_univ _)) (h_x0_min x' (mem_univ _))
    rw [h_expand x', h_eq] at h_min
    have : (1+ρ)/2 * ‖x' - x0‖₂^2 = 0 := by linarith
    rw [mul_eq_zero] at this
    cases this
    · have : (1+ρ)/2 ≠ 0 := by linarith
      contradiction
    · rwa [norm_sq_eq_zero] at this
  exact ⟨x0, h_x0_min, h_unique⟩
\end{lstlisting}
The idea of the model generated prove is to validate $x_0$ is the unique local minima of the target function by using inequalities and transformation of the target function. However, the model fails to generate a correct minima $x_0 = \frac{1}{\rho+1}(c+\rho z -y)$. It mistakenly treats $\frac{1}{\rho+1}(c+\rho (z -y))$ as the minimizer. The proof needs further improvement but admits a good structure to finish.
\paragraph{Conclusion}
The model is capable of generating formal proofs for basic analytical assumptions, such as boundedness and continuity. However, when faced with more intricate properties—such as the KL property or the Lipschitz continuity of the gradient—its ability to produce meaningful and correct proofs remains limited. These challenges highlight current limitations in handling long-chain reasoning and complex functional properties, suggesting the need for further advancement in integrating deeper mathematical understanding into the formalization pipeline. 

\section{More Statistic Results}
In this subsection, we provide additional statistical results, offering a more detailed breakdown by problem class. A detailed comparison of formalization completion rates across various optimization problems is presented in Table \ref{table: big completion}, with evaluations along the following dimensions: definitions, theorems, instances, full file generation, and majority voting. The methods compared include direct generation using DeepSeek-V3 and DeepSeek-R1 and our framework SITA with DeepSeek-V3 and DeepSeek-R1, within each algorithm class,  respectively. The results show that SITA methods achieve higher completion rates than direct generation approaches across all evaluation dimensions and optimization algorithms. In particular, SITA demonstrates an advantage in generating complete formal files, where direct generation methods consistently produce invalid or incomplete outputs. This indicates that the structured decomposition and intermediate-error feedback in SITA are crucial for handling longer or more complex formalization targets. The output of DeepSeek-R1 is better than that of DeepSeek-V3, party due to its better reasoning ability. The majorty voting scores, which provide a semantic measure, further confirm the stability of the SITA-R1 model across problem classes.  
\begin{figure}[htbp]
    \centering
    \includegraphics[width=0.9\linewidth, height=0.25\textwidth]{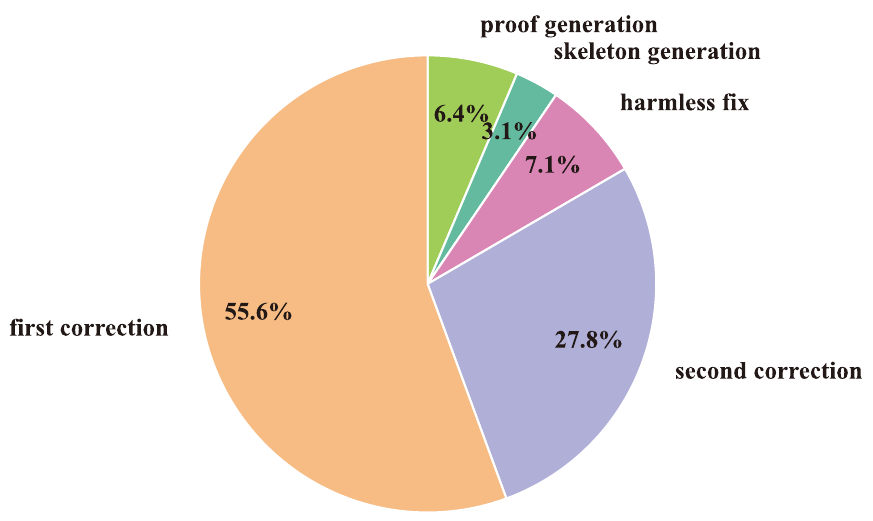}
    \caption{Time consumption of each part of the generation using DeepSeek-R1. The first correction denotes the backbone correction stage. The second correction denotes the proof correction stage.  Harmless fixing denotes the output post-processing part.}
    \label{fig:timeR1}
\end{figure}

Completion rates also vary noticeably across different problem classes. For example, algorithms such as PGM and Nesterov, which have simpler mathematical structures and clearer canonical formulations, result in higher formalization rates. In contrast, algorithms like ADMM, yield lower scores, especially when using direct generation. This reflects the increased difficulty in formalizing problems that involve more intricate dependencies or less commonly seen theoretical constructs. The comparatively lower scores observed for GD can be attributed to the fact that the associated formalization tasks often involve summations or complex expressions with exponential and logarithmic functions, which increase the symbolic and syntactic burden during generation. 

Table \ref{table: pass} presents the performance of different algorithm classes across three evaluation criteria under varying numbers of attempts (pass@1, pass@2, and pass@3). Overall, increasing the number of attempts consistently improves performance across all criteria. The syntactic correctness generally exhibits the highest scores, indicating that the model is relatively reliable at producing syntactically valid definitions and statements. In contrast, proof success remains the most challenging criterion. While generation quality improves with more attempts, challenges remain in semantic correctness and reasoning, particularly for more structurally complex algorithms.
\begin{figure}[htbp]
    \centering
    \includegraphics[width=0.9\linewidth]{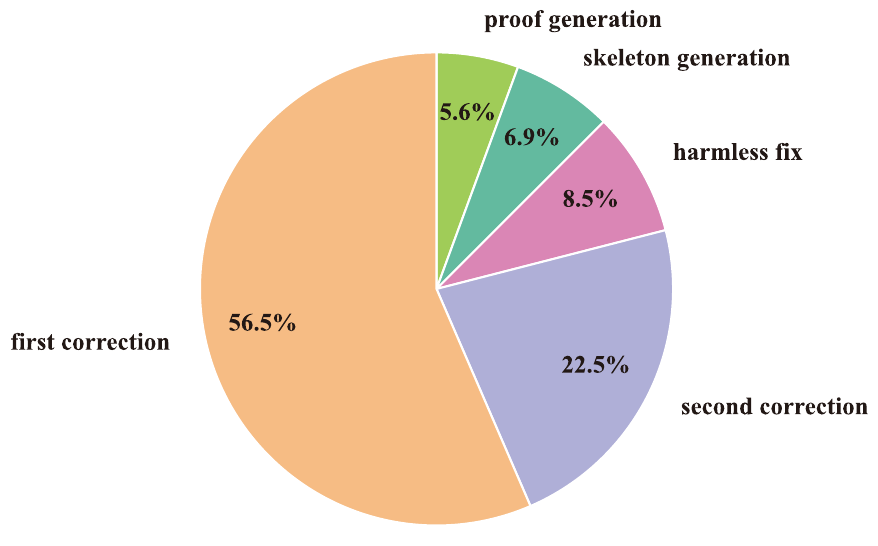}
    \caption{Time consumption of each part of the generation using DeepSeek-V3. The first correction denotes the backbone correction stage. The second correction denotes the proof correction stage. Harmless fixing denotes the output post-processing part.}
    \label{fig:time}
\end{figure}

The complete time allocation for each component is shown in Figure~\ref{fig:timeR1} and Figure~\ref{fig:time} for DeepSeek-R1 and DeepSeek-V3, respectively. Nearly half of the total time is spent on the first round of correction, which focuses on refining the file sketch. Nearly a quarter of the time is devoted to the second correction stage, which targets the proofs. This distribution aligns with the overall success rate of approximately 50\% for DeepSeek-R1 model and 40\% for DeepSeek-V3 model. The second stage of correction focuses on fixing proof-related errors, which are generally longer and more complex than issues in definitions or theorem statements. Hence, this stage can be more time-consuming and demanding. However, we observe that the number of files requiring correction in the second stage is roughly proportional to those in the first stage. This is partly because, although the second stage involves fixing complete proofs, the model sometimes resorts to using sorry, which is disallowed in this stage for proofs to bypass difficult proof obligations. Such behavior reduces the actual effort required from the model, bringing the overall burden closer to that of the first stage, which involves more structural edits such as repairing definitions and statement syntax.

\end{document}